\documentclass{aidata}

\usepackage[utf8]{inputenc}
\usepackage{amsmath,amssymb,amsthm,mathtools}
\usepackage{dsfont}
\usepackage{array,longtable,makecell}
\usepackage{tikz}
\usetikzlibrary{arrows.meta,positioning,calc,fit,shapes.misc,backgrounds}
\usepackage{algorithm}
\usepackage{algpseudocode}
\usepackage{enumitem}
\usepackage{flafter}
\usepackage{textcomp}
\usepackage{listings}
\usepackage{titletoc}
\makeatletter
\def\ttl@gobblecontents#1#2#3#4{\ignorespaces}
\makeatother
\ifdefined\pdfsuppressptexinfo\pdfsuppressptexinfo=-1\fi
\usepackage{url}

\DeclareTextFontCommand{\textbfserif}{\bfseries}
\AtBeginEnvironment{tabular}{\let\textbf\textbfserif}
\AtBeginEnvironment{longtable}{\let\textbf\textbfserif}
\titleformat*{\paragraph}{\bffont\bfseries\sffamily}

\newcommand{\method}{FAULT}
\newcommand{\methodlong}{Self-Diagnosis-guided Terminal Credit Redistribution}

\newcommand{\E}{\mathbb{E}}
\newcommand{\ind}[1]{\mathds{1}\!\left[#1\right]}
\DeclareMathOperator*{\argmin}{arg\,min}
\DeclareMathOperator*{\argmax}{arg\,max}
\DeclareMathOperator{\clip}{clip}
\DeclareMathOperator{\sd}{sd}
\DeclareMathOperator{\sigmoid}{sigmoid}

\DeclareMathOperator{\logdet}{log\,det}
\newcommand{\bbeta}{\boldsymbol{\beta}}
\newcommand{\bx}{\mathbf{x}}
\newcommand{\bw}{\mathbf{w}}
\newcommand{\bfeat}{\mathbf{f}}
\newcommand{\bzero}{\mathbf{0}}
\newcommand{\calD}{\mathcal{D}}
\newcommand{\calC}{\mathcal{C}}

\newcommand{\calM}{\mathcal{M}}
\newcommand{\calY}{\mathcal{Y}}
\newcommand{\calA}{\mathcal{A}}
\newcommand{\calQ}{\mathcal{Q}}
\newcommand{\calV}{\mathcal{V}}
\newcommand{\calU}{\mathcal{U}}
\newcommand{\bdiag}{\mathbf{d}}
\newcommand{\calH}{\mathcal{H}}

\newcommand{\calI}{\mathcal{I}}

\newcommand{\Pq}{\Pr\nolimits_{q}}
\newcommand{\calL}{\mathcal{L}}

\newcommand{\calW}{\mathcal{W}}
\newcommand{\calK}{\mathcal{K}}

\newcommand{\Rout}{R^{\mathrm{out}}}
\newcommand{\dR}{\Delta R}
\newcommand{\Rmin}{R_{\min}}
\newcommand{\Rmax}{R_{\max}}

\newcommand{\Mval}{\calM_{q,i}}

\newcommand{\bref}{\bbeta_{\mathrm{ref}}}
\newcommand{\bhat}{\widehat{\bbeta}}
\newcommand{\bhatv}[1]{\widehat{\bbeta}^{(#1)}}
\newcommand{\bzeroinit}{\bbeta^{(0)}}
\newcommand{\bv}{\bbeta^{(v)}}
\newcommand{\bvm}{\bbeta^{(v-1)}}
\newcommand{\Tfit}{T^{\mathrm{fit}}}

\newcommand{\Sraw}{S^{\mathrm{raw}}}
\newcommand{\Tbar}{\overline{T}^{(0)}}

\newcommand{\Aep}{A^{\mathrm{ep}}}
\newcommand{\Astep}{A^{\mathrm{step}}}
\newcommand{\Zenv}{Z^{\mathrm{env}}}
\newcommand{\Zt}{\widetilde{Z}}

\newcommand{\muep}{\mu^{\mathrm{ep}}}
\newcommand{\mustep}{\mu^{\mathrm{step}}}

\newcommand{\piold}{\pi_{\mathrm{old}}}
\newcommand{\piref}{\pi_{\mathrm{ref}}}
\newcommand{\DLV}{D^{\mathrm{LV}}}

\newcommand{\paragraphhd}[1]{\vspace{2pt}\noindent\textbf{#1}}

\theoremstyle{plain}
\newtheorem{proposition}{Proposition}

\theoremstyle{remark}

\definecolor{PromptFrame}{HTML}{34495E}
\definecolor{PromptBg}{HTML}{F6F7F9}
\definecolor{PromptKey}{HTML}{C0392B}
\definecolor{PromptId}{HTML}{1E8449}

\providecommand{\pk}[1]{{\ttfamily\bfseries\color{PromptKey}#1}}
\providecommand{\pid}[1]{{\ttfamily\bfseries\color{PromptId}#1}}

\lstdefinestyle{promptstyle}{%
  basicstyle=\ttfamily\scriptsize,
  breaklines=true, breakindent=0pt, breakautoindent=false,
  columns=fullflexible, keepspaces=true, upquote=true,
  escapeinside={(*}{*)},
  frame=single, framerule=0.6pt, framesep=4pt,
  rulecolor=\color{PromptFrame}, backgroundcolor=\color{PromptBg},
  xleftmargin=1pt, xrightmargin=1pt, aboveskip=0pt, belowskip=2pt,
}
\newcommand{\prompttitle}[1]{%
  \par\vspace{4pt}\noindent
  \colorbox{PromptFrame}{\parbox{\dimexpr\linewidth-2\fboxsep\relax}{%
    \color{white}\bfseries\footnotesize\strut #1}}%
  \par\nointerlineskip}

\title{My \method{}: Self-Diagnosis as Credit Assignment in\\Self-Evolving Agentic Reinforcement Learning}

\author[1,2,*]{Yihua Zhu}
\author[3,\dagger]{Qianying Liu}
\author[1,\dagger]{Weixu Qiao}
\author[1]{Xuan Ren}
\author[1,4]{Weiwei Xu}
\author[1]{Wenbo Li}
\author[1]{Wei Wang}
\author[1,5]{Ruijia Chen}
\author[1,6]{Xinmiao Luan}
\author[1,7]{Yin Luo}
\author[1,8]{Hao Huang}
\author[1,9]{Xiang Zheng}
\author[2]{Hidetoshi Shimodaira}

\affiliation[1]{Alibaba Group}
\affiliation[2]{Kyoto University}
\affiliation[3]{NII LLMC}
\affiliation[4]{Peking University}
\affiliation[5]{University of California, Los Angeles}
\affiliation[6]{Arizona State University}
\affiliation[7]{The Chinese University of Hong Kong, Shenzhen}
\affiliation[8]{Tsinghua University}
\affiliation[9]{University of the Chinese Academy of Sciences}

\contribution[*]{Work done during an internship at Alibaba AI Data}
\contribution[\dagger]{Corresponding authors}

\abstract{
Agentic reinforcement learning (RL) has emerged as a powerful approach for training large language model agents on multi-step tasks, yet reliance on terminal outcome rewards creates two credit-assignment problems, particularly in long-horizon tasks.
First, same-outcome rollout groups provide no learning signal from terminal rewards. Second, terminal rewards provide only trajectory-wide feedback, making it difficult to identify which decisions caused a failure.
Recent work supplements terminal rewards with finer-grained information from trajectory analysis, such as natural-language reflections on intermediate decisions and errors. 
However, natural-language diagnoses are difficult to use directly for credit assignment: their error claims may be unreliable, and they do not quantify how much each error should affect learning.
We propose Sel{\bfseries f}-Di{\bfseries a}gnosis-g{\bfseries u}ided Termina{\bfseries l} Credi{\bfseries t} Redistribution (\method{}), which turns diagnosed errors into explicit step-level credit anchored by terminal outcomes.
\method{} checks diagnostic evidence and learns relative error costs from task outcomes.
During training, the policy and self-diagnoser co-evolve, while error costs are updated online from recent outcomes.
On ALFWorld, \method{} recovers learning signals from same-outcome groups, reaching 95\% signal coverage versus 41\% for GRPO and 72\% for GiGPO, while better localizing credit to specific error steps.
Across two model scales, \method{} delivers strong improvements on the long-horizon ALFWorld and WebShop tasks while remaining competitive on short-horizon Search-based QA.
}

\date{\today}
\correspondence{Qianying Liu at \email{ying@nii.ac.jp}, Weixu Qiao at \email{qiaoweixu.qwx@alibaba-inc.com}}
\checkdata[Code]{\url{https://github.com/SKYLENAGE-AI/FAULT_Agentic_RL}}

\begin{document}

\maketitle

\begin{figure}[t]
\centering
\includegraphics[width=\textwidth]{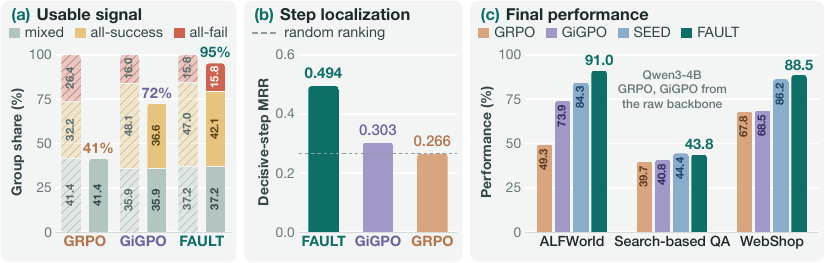}
\caption{\textbf{Signal coverage, step localization and final performance.}
\textbf{(a)~Signal coverage.} On ALFWorld with shared diagnosis-SFT initialization, hatched bars show the group mix and solid bars the groups with usable signal. \method{} uses $95\%$ of groups, including every all-fail group and most all-success groups. GRPO uses only mixed groups ($41\%$); GiGPO also uses some all-success groups but no all-fail groups ($72\%$). Thus, \method{} can learn from differences in intermediate behavior even when terminal rewards are identical, making both successful and failed rollouts useful beyond their final outcomes.
\textbf{(b)~Localization.} Mean reciprocal rank of the decisive error step a blind LLM judge marks in each failed trajectory, with steps ranked by each method's penalty. \method{} ($0.494$) nearly doubles random ($0.266$) while GiGPO ($0.303$, same-configuration rerun) stays close to it, so \method{}'s step-level credit is far more accurate. The gains highlight the value of diagnostic evidence in identifying which decisions within a failed trajectory need correction and directing step-level feedback to those decisions.
\textbf{(c)~Outcome.} Qwen3-4B results from Table~\ref{tab:main}. \method{} leads in ALFWorld success rate and WebShop task score on these longer-horizon tasks, remains competitive with SEED on shorter-horizon Search-based QA, and outperforms GRPO and GiGPO initialized from the raw backbone on all three.}
\label{fig:teaser}
\vspace{-4pt}
\end{figure}

\section{Introduction}
\label{sec:intro}

Large language model agents solve increasingly complex tasks through sequences of decisions: they reason, interact with external environments, observe feedback, and adapt their subsequent actions~\citep{yao2023react}. Agentic reinforcement learning improves such behavior by optimizing these interaction trajectories, typically using rewards based on the final task outcome~\citep{zhang2025landscape}. Terminal outcome supervision is widely used~\citep{jin2025searchr1, qian2025toolrl,feng2025retool, wei2025webagentr1, feng2025gigpo} because final outcomes are often easier to verify than intermediate decisions, but it creates a fundamental credit-assignment problem: a single reward must supervise many intermediate decisions that may have contributed very differently to the final result. As trajectories grow longer, the reward becomes increasingly distant from the decisions it is meant to improve.

This limitation is especially pronounced in group-relative RL methods such as GRPO~\citep{shao2024deepseekmath,guo2025deepseekr1}, which normalize terminal rewards across multiple rollouts of the same task, and broadcast each rollout-level advantage to all of its steps. This produces two distinct credit-assignment limitations. First, when all rollouts in a group receive the same terminal outcome, their normalized advantages vanish, so the group provides no outcome-driven learning signal, even if the trajectories differ substantially in their intermediate decisions. In our preliminary ALFWorld~\citep{shridhar2021alfworld} runs, such same-outcome groups account for 59–64\% of training groups. Second, even when a group contains both successes and failures, the resulting trajectory-level advantage does not identify which intermediate decisions were responsible for the outcome: all steps within a rollout receive the same signal. As trajectories become longer, this makes it increasingly difficult to distinguish decisive errors from otherwise reasonable actions.

A natural way to address these limitations is to introduce process-level signals that provide finer-grained feedback than terminal rewards~\citep{lightman2024verify,wang2024mathshepherd,setlur2024rewarding}. One line of work constructs numerical signals for intermediate steps, using learned reward models~\citep{xi2025agentprm}, Monte Carlo estimates~\citep{choudhury2025agentprm}, hindsight-based scores~\citep{ma2026repro,yao2026mise}, or comparisons across related states~\citep{feng2025gigpo}. However, many such signals are still derived from terminal outcomes and therefore lose discriminative power when outcomes are identical. A complementary line of work analyzes trajectories in natural language, using reflection~\citep{shinn2023reflexion,zhao2024expel,zheng2026grsd,wu2026seed}, hindsight~\citep{bensal2025reflect}, or error diagnosis to identify what went wrong and where~\citep{yuan2025agentr,zhang2025critiquegrpo}. These analyses provide richer information about intermediate decisions, but remain descriptive: they are not directly expressed as quantitative step-level credit, and their error claims require heavy verification before being used for learning.

\begin{figure}[!t]
\centering
\includegraphics[width=\textwidth]{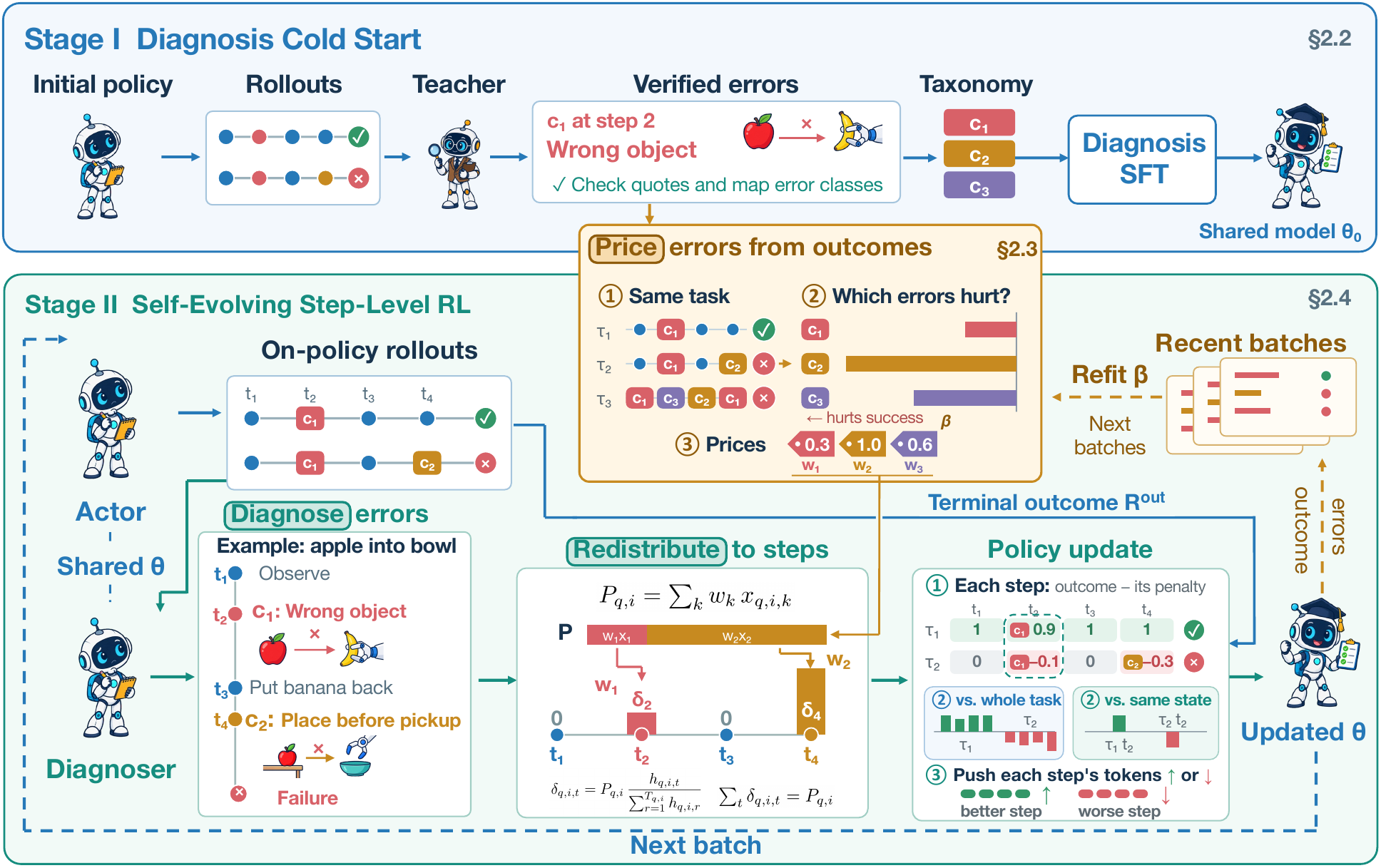}
\caption{\textbf{Overview of \method{}.}
\emph{Stage I} initializes a shared actor and diagnoser by SFT on teacher diagnoses and learns error prices from terminal outcomes.
In \emph{Stage II}, we \textbf{diagnose} errors in on-policy rollouts, \textbf{price} them from terminal outcomes, and \textbf{redistribute} a bounded penalty budget across diagnosed steps for policy updates.
The actor, diagnoser and prices co-evolve through shared model updates and online price refits.}
\label{fig:overview}
\end{figure}

Turning process information into effective credit assignment requires solving two key challenges. First, the process signal must be trustworthy: an error diagnosis can misidentify either what went wrong or where it occurred, causing credit to be assigned to the wrong step. This is a practical concern in agent traces, where even strong models can struggle to jointly identify and localize errors. Second, identifying an error does not determine how much that error should matter for learning. Different error types may have very different effects on task success, and multiple errors within the same trajectory should not deserve equal weight. Thus, useful step-level credit requires both verifiable localization of errors and outcome-grounded calibration of their relative importance.

We propose Sel{\bfseries f}-Di{\bfseries a}gnosis-g{\bfseries u}ided Termina{\bfseries l} Credi{\bfseries t} Redistribution (\method{}), which turns self-diagnosed errors into explicit step-level credit, while retaining terminal outcomes as the source of supervision. The key idea is to separate error localization from error valuation: self-diagnosis identifies what went wrong and where, while terminal outcomes determine how costly different error types are. \method{} first performs \textbf{diagnosing}, which converts trajectory analyses into structured error records and verifies each diagnosis against evidence from the cited step. It then performs \textbf{pricing}, which learns relative error costs from observed task outcomes, and uses these costs to finally \textbf{redistribute} a bounded penalty budget across the diagnosed steps. Because the learned error costs can be applied even when all rollouts receive the same terminal outcome, \method{} can distinguish trajectories that would otherwise receive identical credit, and can concentrate learning signal on the steps most in need of correction.

We summarize how \method{} addresses the two credit-assignment problems and how these improvements translate into downstream performance in Figure \ref{fig:teaser}. First, \method{} substantially increases training signal coverage: on ALFWorld, $95\%$ of rollout groups provide non-negligible credit contrast, compared with $41\%$ for GRPO and $72\%$ for GiGPO. Second, \method{} better localizes credit to the decisions responsible for failure, achieving a decisive-step MRR of $0.494$, compared with $0.303$ for GiGPO and $0.266$ for random ranking. Finally, across two model scales, \method{} achieves the strongest improvements on the longer-horizon ALFWorld~\citep{shridhar2021alfworld} and WebShop~\citep{yao2022webshop} tasks while remaining competitive on the much shorter Search-based QA tasks. These results demonstrate the effectiveness of our finer-grained credit assignment becoming more important as trajectories grow longer.  

\section{Method}
\label{sec:method}

We present \methodlong{} (\method{}), a framework for explicit step-level credit assignment in agentic RL.
A cold start initializes a shared actor and diagnoser through diagnosis SFT and learns initial error prices; RL then locates errors through structured self-diagnosis, fits their prices from terminal outcomes, and redistributes a bounded penalty across the diagnosed steps.

\subsection{Problem Formulation}
\label{sec:problem}
An LLM policy $\pi_\theta$ solves task $q\sim\calD$ through repeated interaction with an external world.
At step $t$, it generates response $a_{q,i,t}\sim\pi_\theta(\cdot\mid o_{q,i,t})$ from context $o_{q,i,t}$ containing the task, interaction history and feedback so far; the world executes the response and returns feedback, appended to the context.
Each response-feedback pair forms an \emph{active interaction step}.
Task $q$'s $i$-th rollout is $\tau_{q,i}=(o_{q,i,1},a_{q,i,1},\ldots,o_{q,i,T_{q,i}},a_{q,i,T_{q,i}})$, with $T_{q,i}$ such steps.
Following group-based RL, we sample $G$ rollouts per task from frozen behaviour policy $\piold$.

Each trajectory receives terminal reward $\Rout_{q,i}$ only at termination, giving the outcome objective
\begin{equation}
J_{\mathrm{out}}(\theta)=\E_{q\sim\calD}\,\E_{\tau_{q,i}\sim\pi_\theta}\big[\Rout_{q,i}\big].
\label{eq:objective}
\end{equation}
This reward is the only judgement guaranteed correct; we preserve it and redistribute only its credit across steps.

\subsection{Cold Start: Verifiable Diagnosis and Diagnosis SFT}
\label{sec:m-diag}
Before RL, a one-time cold start fixes the error taxonomy, teaches the policy to produce verifiable diagnoses of its completed trajectories, and estimates initial error prices.
Pricing errors from terminal outcomes (\S\ref{sec:m-price}) and penalizing diagnosed steps (\S\ref{sec:m-rl}) require diagnoses that are countable across trajectories and checkable against each one.
Free-form reflections lack this structure, so we map error claims to fixed classes and check evidence from their cited steps.

\paragraphhd{Teacher diagnosis.}
Teacher model GPT-5.6 Sol~\citep{openai2026gpt56} reads initial-policy trajectories and outcomes, recording each error's label, step and an exact quote from that step.

\paragraphhd{Verification.}
We check each report's step and quote against the trajectory.
Accepted labels form a fixed taxonomy, from which we select $K$ classes with enough examples for pricing: $\calC=\{c_1,\ldots,c_K\}$ (Appendix~\ref{app:coldstart}).
Teacher and later policy diagnoses of $\tau_{q,i}$ are represented as entries $(k,t,e)$: error class, step and quote.
Verified entries of priced class $k$ give count $n_{q,i,k}$ and per-step rate $x_{q,i,k}=n_{q,i,k}/T_{q,i}$; the rates form $\bx_{q,i}$ (Appendix~\ref{app:diag}).

\paragraphhd{Diagnosis SFT.}
For each cold-start example $(q,i)\in\calD_{\mathrm{sft}}$ with an accepted teacher diagnosis, we fine-tune $\pi_\theta$ to reproduce its verified text $\bdiag_{q,i}$ given trajectory $\tau_{q,i}$ and success label $y_{q,i}\in\{0,1\}$:
\begin{equation}
\calL_{\mathrm{SFT}}(\theta)=-\,\mathbb{E}_{(q,i)\in\calD_{\mathrm{sft}}}\sum_{\ell=1}^{|\bdiag_{q,i}|}\log\pi_\theta\big(d_{q,i,\ell}\mid\tau_{q,i},y_{q,i},\bdiag_{q,i,<\ell}\big),
\label{eq:sft-main}
\end{equation}
Only diagnosis tokens are supervised; prompts include the frozen taxonomy, trajectory and outcome (Appendix~\ref{app:coldstart}).
The checkpoint initializes the RL actor and diagnoser; the same verified diagnoses and outcomes initialize prices via the estimator in \S\ref{sec:m-price}.
No teacher is used thereafter.

\subsection{Terminal-Anchored Error Pricing}
\label{sec:m-price}
Locating an error does not reveal its cost, so \method{} regresses terminal outcomes on verified error rates (\S\ref{sec:m-diag}) to estimate an outcome association coefficient $\beta_k$ for each class.
A more negative $\beta_k$ links higher error rates to worse outcomes, holding other features fixed.
We convert these coefficients to nonnegative relative error prices $w_k$ (\S\ref{sec:m-rl}).

Tasks differ in difficulty, so we compare successful and failed rollouts within each task.
We take the observed number of successes as given and use error rates to explain which rollouts succeed, without separately estimating each task's baseline success rate.
For group $q$, let $\calY_q$ be the successes retained by the diagnosis gate (Appendix~\ref{app:diag}), and $\calA_q$ all candidate success sets $\Gamma$ drawn from all retained trajectories with $|\Gamma|=|\calY_q|$.
This gives the observed success-set probability and likelihood objective over informative groups $\calQ_{\mathrm{info}}$, containing both successes and failures:
\begin{equation}
\Pq(\bbeta)=\frac{\exp\big(\sum_{i\in\calY_q}\bbeta^{\top}\bx_{q,i}\big)}{\sum_{\Gamma\in\calA_q}\exp\big(\sum_{i\in\Gamma}\bbeta^{\top}\bx_{q,i}\big)},\qquad
\bbeta^{\star}=\argmin_{\bbeta}\Big[-\!\!\sum_{q\in\calQ_{\mathrm{info}}}\!\!\log \Pq(\bbeta)\Big].
\label{eq:condlogit}
\end{equation}
For binary outcomes, these coefficients are fitted by bounded quasi-Newton optimization with a small-sample correction and shrinkage toward a fixed anchor: initially on teacher-diagnosed trajectories (\S\ref{sec:m-diag}), then periodically on a sliding window of recent RL rollouts (derivation, estimator, acceptance test and update rule in Appendix~\ref{app:pricing}).

\subsection{Self-Evolving RL with Step-Level Credit Redistribution}
\label{sec:m-rl}
Verified errors can differ even when all rollouts in a group succeed or fail, and can provide step-specific process signals when priced.
In each batch, a frozen policy generates and diagnoses a rollout group per task.
Using fixed error prices, we set each trajectory's bounded total penalty, redistribute it to diagnosed steps, and use each step's share to form group-relative advantages for policy learning.

\paragraphhd{Self-diagnosis of on-policy rollouts.}
After each rollout, the actor diagnoses its completed trajectory given the observed outcome, using the format of \S\ref{sec:m-diag}.
Verified errors are aggregated into error-rate vector $\bx_{q,i}$.

\paragraphhd{Bounded budget.}
\label{sec:m-budget}
We cap the total penalty below the success--failure reward gap for binary outcomes, so steps from successful trajectories retain higher episode scores than those from failures.
Before this cap, fixed prices turn each trajectory's verified diagnosis into a score $S_{q,i}\in[0,1]$ and budget $P_{q,i}$:
\begin{equation}
P_{q,i}=\lambda_u S_{q,i}=\lambda_u\,\clip\!\Big(\sum_{k}w_k\,x_{q,i,k},\,0,\,1\Big),\quad
w_k=\begin{cases}|\beta_k|/\max_{j:\,\beta_j<0}|\beta_j|, & \beta_k<0,\\[1pt] 0, & \beta_k\ge0.\end{cases}
\label{eq:score}
\end{equation}
Here $x_{q,i,k}=n_{q,i,k}/T_{q,i}$ is the verified error rate, and $w_k$ is the nonnegative error price derived from $\beta_k$ (\S\ref{sec:m-price}).
Using rates rather than counts avoids rewarding the RL actor for shorter trajectories solely because longer ones accumulate more diagnosed errors.
$\lambda_u$ sets batch $u$'s penalty strength; Appendices~\ref{app:pricing} and~\ref{app:budget} detail the price mapping, schedule and cap.

\paragraphhd{Conserved redistribution.}
\label{sec:m-redis}
The budget is split in proportion to step localization weights $h_{q,i,t}$,
\begin{equation}
\delta_{q,i,t}=P_{q,i}\,\frac{h_{q,i,t}}{\sum_{r=1}^{T_{q,i}}h_{q,i,r}},\qquad
h_{q,i,t}=\sum_{e\in\calV_{q,i,t}}w_{k(e)},
\label{eq:redistribution}
\end{equation}
Here $\calV_{q,i,t}$ contains verified errors cited at step $t$, and $k(e)$ is entry $e$'s class.
Their summed prices give the raw step penalty weight $h_{q,i,t}$; normalization to $P_{q,i}$ gives $\delta_{q,i,t}$, the penalty allocated to step $t$ for the policy update (exact form in Appendix~\ref{app:redis}).

\paragraphhd{Dual-channel advantage.}
\label{sec:m-adv}
Task-wide comparisons provide overall outcome guidance but use a baseline shared across different contexts.
We therefore add comparisons within similar contexts to give each decision a more relevant reference.
Following GiGPO~\citep{feng2025gigpo}, the episode channel compares each active step with $\Omega_q$, all task-$q$ step samples $(j,r)$; the step channel uses the anchor-state group $\calH(o_{q,i,t})\subseteq\Omega_q$ of samples with the same or similar pre-response context (Appendix~\ref{app:adv}).
Both channels use the same step penalty; group centering cancels any penalty shared by all its samples.
The episode channel centers terminal reward minus the step's share within the task group,
\begin{equation}
\Aep_{q,i,t}=E_{q,i,t}-\frac{1}{|\Omega_q|}\!\sum_{(j,r)\in\Omega_q}\!E_{q,j,r},
\qquad
E_{q,i,t}=\Rout_{q,i}-\delta_{q,i,t},
\label{eq:aep}
\end{equation}
and the step channel centers discounted return minus the step penalty within the anchor-state group,
\begin{equation}
\Astep_{q,i,t}=\Zt_{q,i,t}-\frac{1}{|\calH(o_{q,i,t})|}\!\sum_{(j,r)\in\calH(o_{q,i,t})}\!\Zt_{q,j,r},
\qquad
\Zt_{q,i,t}=\gamma^{\,T_{q,i}-t}\Rout_{q,i}-\delta_{q,i,t},
\label{eq:astep}
\end{equation}
\paragraphhd{Policy update.}
The joint step advantage $A_{q,i,t}=\Aep_{q,i,t}+\omega\Astep_{q,i,t}$, with step-channel weight $\omega$, is broadcast uniformly to step $t$'s response tokens $\calU_{q,i,t}$.
The actor minimizes the clipped policy-gradient objective over all $N$ response tokens in the batch:
\begin{equation}
\calL_{\mathrm{actor}}(\theta)=-\frac{1}{N}\sum_{(q,i,t)}\ \sum_{\ell\in\calU_{q,i,t}}\Big[\min\big(\varrho_\ell A_{q,i,t},\ \clip(\varrho_\ell,1-\epsilon,1+\epsilon)\,A_{q,i,t}\big)-\lambda_{\mathrm{KL}}\,\DLV_\ell\Big],
\label{eq:actor-main}
\end{equation}
where $\varrho_\ell=\pi_\theta(a_\ell\mid c_\ell)/\piold(a_\ell\mid c_\ell)$ is token $\ell$'s importance ratio against the rollout policy given context $c_\ell$, and $\DLV_\ell$ is a low-variance KL estimate to frozen reference policy $\piref$ at that token (Appendix~\ref{app:adv}).

\paragraphhd{Co-evolution of diagnoser, prices and policy.}
\label{sec:m-online}
The diagnoser shares the actor's weights, but its diagnosis tokens receive no policy gradient, preventing a direct RL shortcut of reporting fewer errors to avoid penalties.
After each actor update, the batch's verified records $(q,\bx_{q,i},y_{q,i})$ enter a sliding window; with enough informative groups, prices are periodically refitted online via Eq.~\eqref{eq:condlogit} with the regularization and smoothing of Appendix~\ref{app:online}.
The updated model generates and diagnoses new rollouts, and the latest prices turn these diagnoses into step penalties for subsequent policy updates, closing the loop (Algorithm~\ref{alg:dtcr} in Appendix~\ref{app:algorithm}).

\section{Experiments}
\label{sec:exp}
\subsection{Experimental Setting}
\label{sec:exp-setting}
We evaluate \method{} on three benchmarks. Further details are in Appendix~\ref{app:exp}.

\paragraphhd{Benchmarks.}
ALFWorld~\citep{shridhar2021alfworld} spans six text-based household task families; we train on standard games and test on unseen tasks within $30$ steps.
In WebShop~\citep{yao2022webshop}, the agent searches $1{,}000$ products and buys one matching a natural-language request within $15$ steps.
Search-based QA follows Search-R1~\citep{jin2025searchr1}: the agent searches for evidence to answer questions from NQ~\citep{kwiatkowski2019nq}, TriviaQA~\citep{joshi2017triviaqa}, PopQA~\citep{mallen2023popqa}, HotpotQA~\citep{yang2018hotpotqa}, 2WikiMultiHopQA~\citep{ho2020twowiki}, MuSiQue~\citep{trivedi2022musique} and Bamboogle~\citep{press2023bamboogle}.

\paragraphhd{Baselines.}
Vanilla is the backbone without task-specific training.
GRPO~\citep{shao2024deepseekmath} uses group-normalized outcome rewards and GiGPO~\citep{feng2025gigpo} adds step credit from anchor-state comparisons; both start from the raw backbone.
SEED~\citep{wu2026seed} uses its own teacher-written hindsight-skill SFT data and distils skill-conditioned behaviour during RL.
At 4B, GRPO (diag-SFT) and GiGPO (diag-SFT) start from our diagnosis-SFT checkpoint without diagnostic shaping.

\paragraphhd{Evaluation metrics.}
ALFWorld reports success by task family and over all $134$ unseen tasks, averaging three greedy evaluations per task.
WebShop reports mean normalized task score and strict purchase success over $256$ episodes.
Search-based QA reports exact match for each of seven datasets and the unweighted mean of their accuracies, not a pooled question-level average.

\paragraphhd{Implementation details.}
We use Qwen3-1.7B (thinking-capable) and Qwen3-4B-Instruct-2507 (non-thinking)~\citep{yang2025qwen3,qwen2025qwen34binstruct2507}, with the same no-thinking evaluation template.
For ALFWorld and WebShop, cold start samples eight rollouts for each of $180$ tasks per checkpoint.
Verified diagnoses yield the frozen taxonomy, diagnosis-SFT checkpoint and initial prices; the checkpoint initializes actor and diagnoser.
RL runs $160$ updates with $16$ tasks ($128$ for Search-based QA), $G=8$ rollouts and three self-diagnoses per trajectory; prices are refit every $10$ updates from update $30$.

\begin{table}[t]
\centering
\caption{Performance on ALFWorld, Search-based QA, and WebShop. ALFWorld reports success rate (\%) by task family and over all $134$ unseen tasks (All); Search-based QA reports exact-match accuracy (\%) by dataset and its unweighted average (Avg); WebShop reports mean task score and strict purchase success (\%) over $256$ episodes. Rows marked (diag-SFT) start RL from our Stage~1 cold-start model. Within each backbone, \textbf{bold} and \underline{underlined} mark the highest and second-highest reported values per column (ties share formatting).}
\label{tab:main}
\footnotesize
\setlength{\tabcolsep}{3pt}
\resizebox{\textwidth}{!}{%
\begin{tabular}{l ccccccc cccccccc cc}
\toprule
& \multicolumn{7}{c}{ALFWorld} & \multicolumn{8}{c}{Search-based QA} & \multicolumn{2}{c}{WebShop} \\
\cmidrule(lr){2-8}\cmidrule(lr){9-16}\cmidrule(lr){17-18}
Method & Pick & Look & Clean & Heat & Cool & Pick2 & All & NQ & Triv & Pop & Hotp & 2Wk & MuS & Bam & Avg & Score & Succ. \\
\midrule
\multicolumn{18}{c}{\textit{Qwen3-1.7B}} \\
Vanilla & 8.3 & 22.2 & 6.5 & 0.0 & 0.0 & 0.0 & 6.0 & 29.7 & 48.2 & 35.0 & 23.7 & 19.6 & 6.4 & 17.6 & 25.7 & 51.0 & 2.3 \\
GRPO & 66.7 & 38.9 & 35.5 & 43.5 & 52.4 & 17.6 & 43.3 & 39.3 & 57.1 & 43.3 & 35.7 & 33.9 & 11.7 & 28.0 & \underline{35.6} & 66.5 & 39.1 \\
GiGPO & 58.3 & 27.8 & 29.0 & \underline{65.2} & 38.1 & \underline{29.4} & 41.8 & 39.1 & 56.6 & 44.4 & 33.2 & 29.2 & 11.2 & 23.2 & 33.8 & 78.8 & 49.2 \\
SEED & \textbf{79.2} & \textbf{61.1} & \underline{71.0} & \underline{65.2} & \underline{61.9} & \textbf{47.1} & \underline{65.7} & \underline{42.3} & \underline{58.5} & \underline{44.8} & \underline{38.4} & \textbf{36.9} & \textbf{15.4} & \textbf{36.0} & \textbf{38.9} & \underline{86.8} & \textbf{73.4} \\
\textbf{\method{} (Ours)} & \underline{70.8} & \underline{44.4} & \textbf{74.2} & \textbf{78.3} & \textbf{85.7} & \textbf{47.1} & \textbf{68.7} & \textbf{44.6} & \textbf{59.7} & \textbf{46.5} & \textbf{39.8} & \underline{35.4} & \underline{13.7} & \underline{32.8} & \textbf{38.9} & \textbf{87.5} & \underline{69.5} \\
\midrule
\multicolumn{18}{c}{\textit{Qwen3-4B-Instruct-2507}} \\
Vanilla & 50.0 & 11.1 & 16.1 & 13.0 & 19.0 & 0.0 & 19.4 & 4.0 & 11.0 & 4.3 & 5.9 & 6.9 & 1.9 & 8.8 & 6.1 & 20.5 & 1.6 \\
GRPO & 75.0 & 16.7 & 32.3 & 69.6 & 57.1 & 41.2 & 49.3 & 43.7 & 60.6 & 44.8 & 40.7 & 35.0 & 18.7 & 34.4 & 39.7 & 67.8 & 55.1 \\
GiGPO & 75.0 & \underline{66.7} & \underline{83.9} & 78.3 & 81.0 & 47.1 & 73.9 & 42.6 & 62.3 & 44.6 & 39.6 & 42.0 & 16.9 & 37.6 & 40.8 & 68.5 & 50.4 \\
GRPO (diag-SFT) & 79.2 & 61.1 & 77.4 & \underline{82.6} & 76.2 & 41.2 & 71.6 & \textbf{50.4} & \underline{65.5} & \textbf{49.4} & \textbf{48.1} & 42.3 & 18.5 & \underline{45.6} & \underline{45.7} & 82.2 & 66.4 \\
GiGPO (diag-SFT) & 79.2 & \textbf{100.0} & 80.6 & 78.3 & 76.2 & \underline{70.6} & 80.6 & 47.6 & \textbf{65.6} & \underline{48.8} & \underline{45.7} & \textbf{47.0} & \textbf{21.7} & \textbf{48.8} & \textbf{46.5} & 83.5 & 68.0 \\
SEED & \underline{83.3} & \textbf{100.0} & \underline{83.9} & 73.9 & \textbf{95.2} & \underline{70.6} & \underline{84.3} & \underline{48.7} & 63.5 & 47.3 & 43.7 & \underline{43.6} & 20.1 & 44.0 & 44.4 & \underline{86.2} & \textbf{74.2} \\
\textbf{\method{} (Ours)} & \textbf{91.7} & \textbf{100.0} & \textbf{93.5} & \textbf{87.0} & \underline{85.7} & \textbf{88.2} & \textbf{91.0} & 47.9 & 64.1 & 46.8 & 44.3 & 41.8 & \underline{20.4} & 41.6 & 43.8 & \textbf{88.5} & \underline{71.5} \\
\bottomrule
\end{tabular}}
\end{table}

\begin{figure}[t]
\centering
\includegraphics[width=\textwidth]{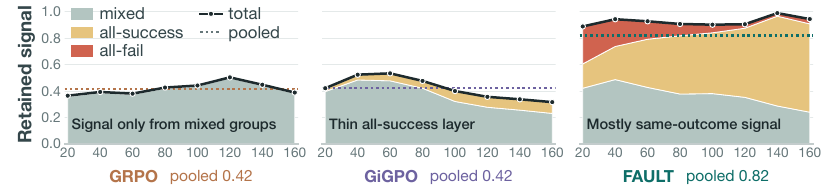}
\caption{\textbf{Relative signal strength during training.} Stacked areas show the contributions of mixed, all-success, and all-fail groups to the signal-retention Index. Black lines show totals for $20$-update windows; dotted lines show run-pooled values. Both baselines start from diagnosis-SFT.}
\label{fig:signal_index_curve}
\end{figure}

\begin{figure}[t]
\centering
\includegraphics[width=\textwidth]{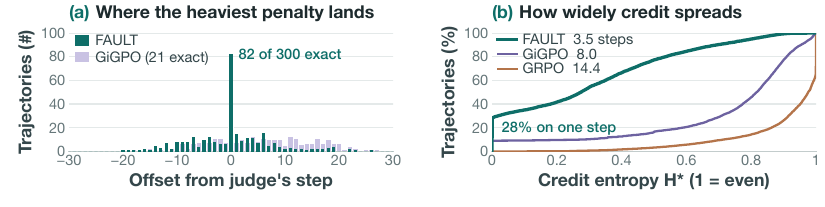}
\caption{\textbf{Placement and concentration of step-level credit.} (a) Offset of the most penalized step from the judge's decisive step over $300$ failed trajectories: \method{} in dark bars, GiGPO (same-configuration rerun) in light bars. (b) Cumulative distribution of the normalized credit entropy on same-outcome groups; curves further left are more concentrated, and the legend gives the mean effective number of credited steps. Baselines start from the diagnosis-SFT checkpoint.}
\label{fig:localization}
\end{figure}

\subsection{Main Results}
\label{sec:exp-main}
Table~\ref{tab:main} reports results for two backbones on three benchmarks.

\paragraphhd{Strong performance across benchmarks and scales.}
At both scales, \method{} has the highest ALFWorld success rate and WebShop task score.
Across scales, its gains over raw-initialized GRPO/GiGPO span $17.1$--$41.7$ points on ALFWorld, $8.7$--$21.0$ in WebShop task score and $3.0$--$5.1$ on the Search-based QA average.
At 4B, it also beats their (diag-SFT) variants on ALFWorld and WebShop, showing gains beyond shared initialization.
Margins over SEED, the strongest ALFWorld/WebShop baseline, are smaller: $6.7$ and $3.0$ points on ALFWorld and $2.3$ and $0.7$ in WebShop task score at 4B and 1.7B.
On Search-based QA, \method{} ties SEED at $38.9$ at 1.7B; at 4B, it scores $43.8$, $0.6$ points below SEED.
SEED leads WebShop strict success by $2.7$ and $3.9$ points at 4B and 1.7B, likely because \method{}, unlike the baselines, trains on the graded score that also credits partial matches, and SEED's distilled skills suit this all-or-nothing metric.

\paragraphhd{Initialization and RL both contribute.}
The 4B (diag-SFT) rows distinguish cold-start and RL gains, though they are imperfect single-factor controls (Appendix~\ref{app:exp}).
Our diagnosis-SFT checkpoint raises GRPO/GiGPO from $49.3/73.9$ to $71.6/80.6$ on ALFWorld, $67.8/68.5$ to $82.2/83.5$ in WebShop score, and $39.7/40.8$ to $45.7/46.5$ on Search-based QA.
GiGPO (diag-SFT) still trails SEED by $3.7/2.7$ points on ALFWorld/WebShop.
From the same checkpoint, \method{} beats GiGPO (diag-SFT) by $10.4$ points and SEED by $6.7$ on ALFWorld; its $88.5$ WebShop score also exceeds SEED's, though with a different reward.
On Search-based QA, however, \method{} trails these two (diag-SFT) controls by $1.9/2.7$ points, so diagnosis-guided RL's additional gains concentrate in ALFWorld and WebShop.

\paragraphhd{Margins and episode length.}
The above GRPO/GiGPO margin ranges shrink from ALFWorld to WebShop to Search-based QA, matching their step-limit order ($30$, $15$ and $4$).
The same order holds against GiGPO (diag-SFT), with 4B margins of $10.4$, $5.0$ and $-2.7$ points, and against SEED at both scales.
This pattern is consistent with greater benefits from step-level credit on longer tasks: longer failed episodes have more possible error steps, whereas in episodes of at most four steps, outcome rewards already fall close to the wrong step.

\subsection{Analysis}
\label{sec:exp-analysis}
We analyze signal coverage and strength, step-level credit, and training trajectory length.

\subsubsection{Learning Signal in Same-Outcome Groups}
\label{sec:analysis-signal}
\paragraph{Signal coverage.}
On ALFWorld, each prompt's eight rollouts form a \emph{same-outcome} group (all-fail or all-success) or a \emph{mixed} group.
Same-outcome groups make up $59$--$64\%$ of training groups, rising to about $75\%$ by the end of GiGPO and \method{} training.
To compare scales, we measure each group's credit gap, its highest minus lowest credit value.
Mixed groups provide outcome contrast, so a group is usable if its gap exceeds $5\%$ of that method's mean mixed-group gap, allowing weak but non-negligible signals (Appendix~\ref{app:signal-figs}).
From \method{}'s diagnosis-SFT checkpoint, GRPO, GiGPO and \method{} find usable signal in $41.4\%$, $72.5\%$ and $95.1\%$ of all groups (solid bars, Figure~\ref{fig:teaser}(a)).

\paragraph{Signal strength.}
In each $20$-update window, we average the credit gaps of all-success, all-fail and mixed groups separately.
Dividing each mean by the method's mixed-group mean in that window expresses signal strength relative to mixed groups.
We then weight each ratio by that group type's share of the window's groups.
The three contributions sum to the signal-retention Index; mixed groups contribute exactly their share, as their ratio is one (Appendix~\ref{app:signal-figs}).
From the first to last window, \method{}'s all-success share rises from $11.6\%$ to $70.3\%$, while its mixed and all-fail shares fall from $42.2\%$ to $24.1\%$ and $46.3\%$ to $5.6\%$.
Its all-success contribution rises from $0.184$ to $0.668$, keeping the total near $0.9$ (Figure~\ref{fig:signal_index_curve}).
At $72.5\%$ coverage, GiGPO's pooled Index is $0.421$ (same-outcome contribution $0.063$), near GRPO's $0.419$, versus \method{}'s $0.821$ (same-outcome contribution $0.449$).

\paragraph{Why outcome credit falls short.}
All mixed groups qualify, but same-outcome groups expose the limits of outcome credit.
GRPO uses none because subtracting the group mean cancels equal rewards.
GiGPO uses about $75\%$ of all-success but no all-fail groups: discounting separates rollouts by steps left to success from shared anchor states, while failures return zero.
\method{} uses $90\%$ and $100\%$, respectively, as its diagnostic penalties distinguish trajectories despite equal outcomes (derivations in Appendix~\ref{app:signal-figs}).
Successes can still contain correctable process errors: the audit finds $61.85\%$ of \method{}'s successes ($8{,}719/14{,}097$) penalized while keeping their success labels (Appendix~\ref{app:audit}).
The strength comparison uses budget gaps for \method{} and advantage gaps for baselines, so the pooled $0.821$ is a contrast proxy, not a common-scale advantage estimate (Appendix~\ref{app:signal-figs}).

\subsubsection{Step-Level Localization of Credit}
\label{sec:analysis-local}
Section~\ref{sec:analysis-signal} asked whether each group keeps a training signal.
Here we ask how this signal is split across a trajectory's steps and whether it falls on the right steps.

\begin{figure}[t]
\centering
\includegraphics[width=\textwidth]{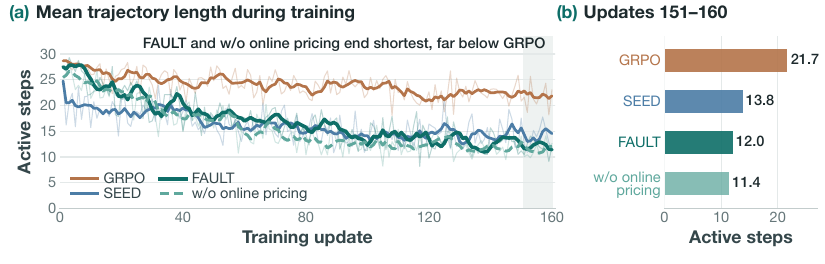}
\caption{\textbf{ALFWorld training trajectory length.} (a)~Mean active steps per training update (thin lines) and five-update moving averages (thick lines); shading marks updates $151$--$160$. (b)~Means over updates $151$--$160$. All runs use a $30$-step limit but differ in initialization and training recipe.}
\label{fig:traj_length}
\end{figure}

\paragraph{Credit placement.}
An independent LLM judge reads each sampled failed training trajectory and marks its decisive error step, the earliest step whose correction would most likely turn the failure into a success.
Each method ranks a trajectory's steps by the penalty it gave them.
Figure~\ref{fig:teaser}(b) reports the mean reciprocal rank of the judge's step, $\mathrm{MRR}=\frac{1}{N}\sum_{i=1}^{N}1/\mathrm{rank}_i$ (allowing one step of error), and Figure~\ref{fig:localization}(a) shows the offset of the most penalized step from the judge's step (Appendix~\ref{app:localization}).
\method{} reaches an MRR of $0.494$ versus $0.266$ for random ranking, and its most penalized step is exactly the judge's step in $82$ of $300$ trajectories.
GiGPO reaches $0.303$ with $21$ exact matches, and its most penalized step usually comes several steps after the judge's; GRPO stays close to random.

\paragraph{Credit concentration.}
The normalized entropy measures how evenly a trajectory's credit is spread over its steps.
We normalize each step's absolute credit (penalty shares for \method{}, advantages for baselines) into a distribution and divide its entropy by the log of the step count, so it is $1$ for an even spread and $0$ when all credit is on one step (Appendix~\ref{app:localization}).
On same-outcome groups, the mean normalized entropy is $0.29$ for \method{} versus $0.72$ for GiGPO and $0.91$ for GRPO, and $28\%$ of \method{}'s penalized trajectories put the whole penalty on one step (Figure~\ref{fig:localization}(b)).

\paragraph{Why step credit differs.}
GRPO gives every step of a trajectory the same credit.
GiGPO's credit varies across steps but comes from comparing returns at states shared by several rollouts: a failed rollout is penalized at every state that a successful rollout of the same task also reached, and in all-fail groups this difference disappears (Proposition~\ref{prop:same-outcome}).
\method{} penalizes only the steps cited by verified diagnoses.
This explains both results: equal credit keeps GRPO near random and the most spread out, GiGPO's state-based credit is more concentrated but rarely hits the decisive error, and \method{}'s credit is both the most concentrated and the best placed.
Appendices~\ref{app:localization}, \ref{app:mech-probe} and~\ref{app:case} give more results.

\subsubsection{Training Trajectory Length}
\label{sec:analysis-length}
Figure~\ref{fig:traj_length} tracks the mean number of active steps per ALFWorld training episode, counting successes and failures.
GRPO ends near $22$ steps, while SEED, \method{} and the variant without online pricing (Section~\ref{app:ablation-details}) become much shorter, averaging $13.8$, $12.0$ and $11.4$ steps over updates $151$--$160$; \method{} is shorter than GRPO and SEED in every window of the final $5$ to $40$ updates (Appendix~\ref{app:length-dynamics}).
Two effects explain this: failed episodes run to the $30$-step limit, so higher training success lowers the mean, and \method{} penalizes steps that waste this budget, such as revisits to searched locations, which recede faster under \method{} than under GRPO (diag-SFT; Appendices~\ref{app:dose-response}--\ref{app:prices-move}).

\subsection{Ablation Studies}
\label{app:ablation-details}
\paragraphhd{Ablation variants.}
We compare \method{} with four design variants on ALFWorld with Qwen3-4B-Instruct-2507 (Table~\ref{tab:ablation}); several differ in more than their named components.
\emph{W/o online pricing} keeps cold-start prices fixed throughout RL while still fitting monitoring coefficients, and removes the cap in Eq.~\eqref{eq:cap}.
\emph{Cosine $\lambda\to$ constant $\lambda$} fixes the penalty strength instead of letting it rise from a low value to a peak and decay to zero.
\emph{Shared diagnoser $\to$ fixed judge} takes diagnoses from the frozen diagnosis-SFT checkpoint; this recorded variant also uses constant $\lambda$, online prices and the cap, so it does not isolate the judge.
\emph{GRPO + diagnostic penalty} is an earlier design that broadcasts a trajectory-level diagnostic penalty to every step rather than targeting diagnosed steps; it also changes score mixing and penalty scale, and experienced a refit failure.
The variant without online pricing and the full method come from later runs than the other variants, so these runs are not matched controls.

\begin{table}[!htbp]
\centering
\caption{Ablations and design variants on ALFWorld with Qwen3-4B-Instruct-2507: success rate (\%) on the $134$ unseen tasks, evaluated as in Table~\ref{tab:main}.}
\label{tab:ablation}
\footnotesize
\setlength{\tabcolsep}{4pt}
\begin{tabular}{@{}l ccccccc@{}}
\toprule
& \multicolumn{7}{c}{ALFWorld} \\
\cmidrule(lr){2-8}
Variant & Pick & Look & Clean & Heat & Cool & Pick2 & All \\
\midrule
\method{} & \textbf{91.7} & \textbf{100.0} & \textbf{93.5} & 87.0 & 85.7 & 88.2 & \textbf{91.0} \\
\midrule
w/o online pricing & 79.2 & \textbf{100.0} & 83.9 & \textbf{95.7} & \textbf{90.5} & \textbf{94.1} & 89.6 \\
Cosine $\lambda$ $\to$ constant $\lambda$ & 70.8 & \textbf{100.0} & 77.4 & 73.9 & 85.7 & 70.6 & 79.1 \\
Shared diagnoser $\to$ fixed judge & \textbf{91.7} & 94.4 & 77.4 & 78.3 & 81.0 & 76.5 & 82.8 \\
GRPO + diagnostic penalty & 79.2 & \textbf{100.0} & 71.0 & 87.0 & 85.7 & 41.2 & 77.6 \\
\bottomrule
\end{tabular}
\end{table}

\paragraphhd{Ablation results.}
With actor, diagnoser, and prices co-evolving, \method{} achieves the best overall success ($91.0\%$), above fixed prices/no cap ($89.6\%$), constant $\lambda$ ($79.1\%$), fixed judge ($82.8\%$), and earlier GRPO-based diagnostic shaping ($77.6\%$).
These comparisons favor the full design but assess recorded design variants rather than isolating localization, pricing or the judge one factor at a time.

\subsection{Diagnostic Quality Before and During RL}
\label{app:diag-ability}
\paragraphhd{Assessment protocol.}
We assess teacher diagnoses and taxonomy coverage on the $2{,}880$ ALFWorld cold-start trajectories collected with Qwen3-4B-Instruct-2507, and evaluate the diagnosis-SFT checkpoint on the $200$ stratified held-out trajectories (Appendix~\ref{app:inst}).
Separately, we test for degradation during RL as $\lambda$ anneals using $17$ frozen checkpoints at updates $0$--$160$ and $n{=}600$ diagnoses per checkpoint.
This assessment uses single-seed training rollouts and a fixed reference from the original teacher's diagnoses, unchanged throughout training.

\paragraphhd{Cold-start results.}
Evidence verification retains $13{,}820$ of $14{,}484$ teacher error mentions ($95.4\%$); every rejected mention fails the quote check.
Successful and failed trajectories contain an average of $3.4$ and $5.1$ verified errors, respectively, showing that the teacher reports process errors even in successful trajectories.
The frozen taxonomy covers $90.6\%$ of accepted mentions.
On the held-out set, the initial self-diagnoser achieves L1 precision $0.81$ and recall $0.56$, above the respective thresholds of $0.75$ and $0.50$.
Dominant-error step accuracy is $0.48$ under a $\pm2$-step tolerance, below the $0.60$ threshold, and the off-table rate is $0.31$, above the $0.10$ limit.
Only precision and recall meet their criteria; step localization and off-table predictions remain limitations.

\begin{figure}[!htbp]
\centering
\includegraphics[width=\textwidth]{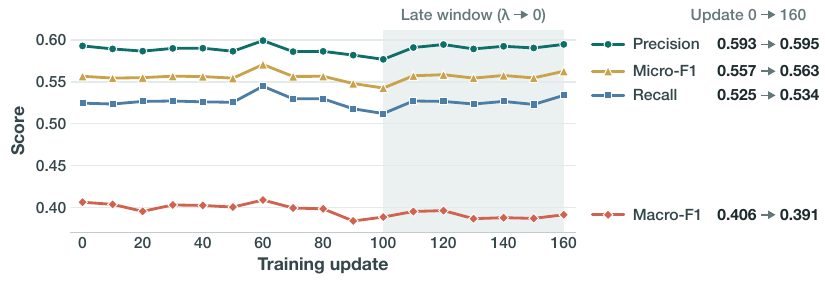}
\caption{\textbf{Diagnostic quality during RL.} Concrete-L2 precision, recall, micro-F1 and macro-F1 against the fixed teacher reference. Shading marks the late $\lambda\to0$ window (updates $100$--$160$). Single seed; training rollouts.}
\label{fig:diag_vs_step}
\end{figure}

\paragraphhd{Results during RL.}
L1 error-slot recall increases from $0.66$ to $0.71$ (Spearman $+0.93$), while the off-table hallucination rate decreases from $0.32$ to $0.16$ (Spearman $-0.94$).
Both metrics are strongest in the late $\lambda\!\to\!0$ window (updates $100$--$160$).
L2 micro-F1 remains stable, and L2 precision and recall do not decline through this window (Figure~\ref{fig:diag_vs_step}); L2 macro-F1 decreases slightly ($-0.015$ over the run).
Macro-F1 weights classes equally and is more sensitive to rare classes; stable micro-F1 therefore does not rule out class-specific degradation.

\section{Related Work}
\label{sec:related}

\paragraphhd{Outcome-based agentic RL.}
Agentic RL typically trains LLM policies from terminal rewards~\citep{zhang2025landscape}, often with GRPO~\citep{shao2024deepseekmath}, which gives each rollout's group-normalized outcome advantage to all its steps.
Same-outcome groups therefore yield no outcome-driven gradient~\citep{yu2025dapo,yang2026progpo}.
Agentic variants change how rollouts are sampled or compared~\citep{dong2025arpo,feng2025gigpo}, but still derive credit from outcomes.

\paragraphhd{Process-based agentic RL.}
Step rewards improve on outcome rewards in mathematical reasoning~\citep{lightman2024verify,wang2024mathshepherd}, and agent RL now scores steps with learned models, rollouts or hindsight~\citep{choudhury2025agentprm,xi2025agentprm,kim2026pair,ma2026repro,yao2026mise}.
Learned step rewards invite reward hacking~\citep{guo2025deepseekr1}, and GiGPO's outcome credit vanishes in all-fail groups~\citep{feng2025gigpo}.
Language analyses of finished trajectories can instead guide later attempts~\citep{shinn2023reflexion,zhao2024expel}, correct training data~\citep{yuan2025agentr,zhang2025critiquegrpo} or be distilled into the policy~\citep{wu2026seed}.
Yet the analysis stays text, its error claims are unchecked, and its weight is not learned from outcomes.
\method{} keeps the terminal reward as the trusted signal and closes these gaps with verified diagnoses, outcome-fitted prices and bounded step credit (Appendix~\ref{app:related}).

\section{Conclusion}
\label{sec:conclusion}
Agentic RL has advanced rapidly in training agents for multi-step tasks, yet terminal rewards offer limited signals in same-outcome groups and little guidance on individual errors, particularly in long-horizon tasks.
Natural-language diagnoses provide process information but need verification and quantification.
\method{} addresses this gap by checking diagnostic evidence and learning relative error costs from terminal outcomes to assign explicit step-level credit.
The policy and diagnoser co-evolve during RL.
On ALFWorld, \method{} raises usable signal coverage to $95\%$, versus $41\%$ for GRPO and $72\%$ for GiGPO, and achieves a decisive-step MRR of $0.494$, versus GiGPO's $0.303$.
At both model scales, it leads in ALFWorld success rate and WebShop task score while remaining competitive on shorter-horizon Search-based QA.
Its ALFWorld training trajectories are also shorter than those of GRPO and SEED.

\subsection*{AI use statement}
In this work, we used generative AI tools to design and give feedback on the method and experiments, implement the method, assist in writing proofs, translate text and interpret results.
Generative AI is also part of the method: GPT-5.6 Sol writes the teacher diagnoses for cold-start SFT and organizes free-text error labels into the error taxonomy (Section~\ref{sec:m-diag} and Appendix~\ref{app:coldstart}), Qwen3.8-27B generates the Search-based QA solving trajectories used in cold-start SFT, and Qwen3.8-Max is the blind judge that labels decisive error steps for the localization analysis (Section~\ref{sec:analysis-local}).
All RL training and evaluation tasks come from the public ALFWorld, WebShop and Search-based QA benchmarks.
We have not used generative AI tools to develop theoretical models or conceptual frameworks, formulate mathematical claims, provide critical ingredients for proofs, propose or refine hypotheses, or clean and reformat datasets.
Additionally, we used generative AI tools to create and modify figures, suggest experimental parameters, write and edit code, draft parts of the paper, summarize and identify relevant literature, edit the paper for readability, and suggest the title and keywords.
We have reviewed all AI-assisted work: the authors checked all AI-assisted content and tested all AI-assisted code, and for proofs, the authors proposed the proof ideas, AI tools helped with the derivations, and the authors verified every step.
We take responsibility for the final content of this work, including text, claims or artifacts produced with the aid of generative AI.

\subsection*{Ethics statement}
This work studies credit assignment for language-model agents using established research benchmarks.
We did not recruit human participants or collect private user data.
Our goal is to make credit assignment more transparent by linking step-level learning signals to diagnosed errors and observed task outcomes.

\subsection*{Reproducibility statement}
Section~\ref{sec:method} and Appendix~\ref{app:method} specify the full method, including diagnosis verification, the pricing estimator and the penalty schedule, and Algorithm~\ref{alg:dtcr} in Appendix~\ref{app:algorithm} lists the complete training loop.
Propositions~\ref{prop:order} and~\ref{prop:same-outcome} state their assumptions and are proved in Appendices~\ref{app:budget} and~\ref{app:signal-figs}.
Appendix~\ref{app:exp} describes the benchmarks, baselines, evaluation protocols, cold-start data and initialization, and Table~\ref{tab:hparams} lists the training hyper-parameters.
The diagnosis prompt is given in Appendix~\ref{app:prompts}, and the remaining prompts are released with our code.
All benchmarks, base models and Qwen3.8-27B are public, and GPT-5.6 Sol and Qwen3.8-Max are available through commercial APIs.

\subsection*{Acknowledgments}
{\uchyph=0 
We thank our colleagues and managers at Alibaba AI Data for their generous help during the internship.
This work was partially supported by JST SPRING JPMJSP2110 (YZ); JSPS KAKENHI JP22H05106, JP23K28045, JP26H02495, JST CREST JPMJCR21N3, and JST Moonshot JPMJMS263I (HS); and JSPS KAKENHI JP26K25543 (QL).\par}

\bibliography{references}
\bibliographystyle{unsrtnat}

\clearpage
\beginappendix
\noindent\textbf{Contents}
\startcontents[appendix]
\begingroup
\setlength{\parskip}{0pt}
\hypersetup{linkcolor=black}
\titlecontents{section}[1.6em]{\addvspace{10pt}\bfseries}{\contentslabel{1.6em}}{\hspace*{-1.6em}}{\titlerule*[0.6pc]{.}\contentspage}
\titlecontents{subsection}[4.1em]{\addvspace{4.5pt}}{\contentslabel{2.5em}}{\hspace*{-2.5em}}{\titlerule*[0.6pc]{.}\contentspage}
\printcontents[appendix]{}{1}{\setcounter{tocdepth}{2}}
\endgroup
\clearpage
\section{Notation}
\label{app:notation}
Table~\ref{tab:notation} lists the symbols of \S\ref{sec:method} and the main ones of Appendix~\ref{app:method}; others are defined where used, and Table~\ref{tab:hparams} gives the constants.
Three letters have two meanings: $j$ also indexes classes (as in $\max_j$), $c$ is a class $c_k$ or a token's context $c_\ell$, and $e$ is the quote of an entry $(k,t,e)$ or, in $k(e)$ and $e\in\calV_{q,i,t}$, the entry itself.
Superscripts such as $(u)$, $(v)$, $(0)$, $\mathrm{out}$, $\mathrm{ep}$ and $\mathrm{step}$ mark versions, roles or stages, never powers.
Bold lower-case letters are vectors; calligraphic letters are sets, except the losses $\calL$, the task distribution $\calD$ and the information matrix $\calI$; $\ind{\cdot}$ is the indicator function, $\clip(\cdot,\mathrm{lo},\mathrm{hi})=\min(\max(\cdot,\mathrm{lo}),\mathrm{hi})$, and $\sd(\cdot)$ is the population standard deviation.

\begingroup
\footnotesize
\setlength{\tabcolsep}{4pt}
\renewcommand{\arraystretch}{1.05}
\begin{longtable}{@{}p{3.05cm}p{8.6cm}p{1.7cm}@{}}
\caption{Notation. ``Where'' gives the place where the symbols are first introduced.}\label{tab:notation}\\
\toprule
\textbf{Symbol} & \textbf{Meaning} & \textbf{Where} \\
\midrule
\endfirsthead
\multicolumn{3}{@{}l}{\footnotesize\textit{Table~\ref{tab:notation} (continued)}}\\
\toprule
\textbf{Symbol} & \textbf{Meaning} & \textbf{Where} \\
\midrule
\endhead
\midrule
\multicolumn{3}{r@{}}{\footnotesize\textit{continued on next page}}\\
\endfoot
\bottomrule
\endlastfoot
\multicolumn{3}{@{}l}{\textbf{Indices, sets and constants}}\\[1pt]
$q\sim\calD$ & task and its rollout group; $\calD$: training task distribution & \S\ref{sec:problem} \\
$i,j\in\{1,\ldots,G\}$ & rollout indices in a group; $G$ rollouts per task & \S\ref{sec:problem} \\
$t,r\in\{1,\ldots,T_{q,i}\}$ & active steps; $T_{q,i}$: steps executed, not the step limit & \S\ref{sec:problem} \\
$k\in\{1,\ldots,K\}$ & priced error class; $\calC=\{c_1,\ldots,c_K\}$: frozen priced set & \S\ref{sec:m-diag} \\
$u\in\{1,\ldots,U\}$; $v$ & actor update (batch), $U$ in total; accepted price refit & \S\ref{sec:m-budget} \\
$u_0,\ \Delta u,\ W$ & first refit update, refit interval, window length & App.~\ref{app:online} \\
$\calQ_{\mathrm{info}}$ & groups with at least one success and one failure & \S\ref{sec:m-price} \\
\addlinespace[1pt]
\multicolumn{3}{@{}l}{\textbf{Trajectories, terminal signal and policies}}\\[1pt]
$\tau_{q,i}$; $o_{q,i,t}$, $a_{q,i,t}$ & trajectory $i$ of task $q$; context before response $t$; response $t$ & \S\ref{sec:problem} \\
$y_{q,i}\in\{0,1\}$ & success indicator of $\tau_{q,i}$ & \S\ref{sec:m-diag} \\
$\Rmin,\Rmax,\dR$; $\Rout_{q,i}$ & failure/success rewards, gap; $\Rout_{q,i}=\Rmin+\dR\,y_{q,i}$ if binary & \S\ref{sec:problem} \\
$J_{\mathrm{out}}(\theta)$ & outcome objective (expected terminal reward) & Eq.~\eqref{eq:objective} \\
$\pi_\theta,\ \piold,\ \piref$ & trainable, per-batch behaviour and KL reference policies & \S\ref{sec:problem} \\
\addlinespace[1pt]
\multicolumn{3}{@{}l}{\textbf{Structured diagnosis}}\\[1pt]
$(k,t,e)$; $k(e)$ & entry: frozen-table class, step, verbatim quote; priced class of entry $e$ & \S\ref{sec:m-diag} \\
$M$; $\Mval,\ M_{q,i}$ & diagnoses sampled per trajectory; the parsable ones and their number & App.~\ref{app:diag} \\
$\calV_{q,i,t}$ & verified entries citing step $t$ & Eq.~\eqref{eq:redistribution} \\
$n_{q,i,k}$; $x_{q,i,k}$, $\bx_{q,i}$ & mean verified count of class $k$; rate $n_{q,i,k}/T_{q,i}$; rate vector & \S\ref{sec:m-diag} \\
$k^{*}_{q,i}$ & class of the earliest verified error & App.~\ref{app:diag} \\
$\calD_{\mathrm{sft}}$; $\bdiag_{q,i}$, $d_{q,i,\ell}$; $\calL_{\mathrm{SFT}}$ & SFT trajectories; verified diagnosis text, its $\ell$-th token; SFT loss & Eq.~\eqref{eq:sft-main} \\
\addlinespace[1pt]
\multicolumn{3}{@{}l}{\textbf{Terminal-anchored pricing}}\\[1pt]
$\alpha_q$ & task intercept, removed by conditioning & \S\ref{sec:m-price} \\
$\bbeta=(\beta_1,\ldots,\beta_K)$ & shared error-outcome coefficients; negative ones become prices & \S\ref{sec:m-price} \\
$\calY_q$; $\calA_q$, $\Gamma$ & observed success set; all same-size candidate sets, one member & Eq.~\eqref{eq:condlogit} \\
$\Pq(\bbeta)$; $\bbeta^{\star}$ & conditional probability of $\calY_q$; its unpenalized maximizer & Eq.~\eqref{eq:condlogit} \\
$\bhat(\bref)$; $\bzeroinit$; $\bv$ & fit with reference $\bref$; anchor $\min(\bhatv{0},0)$; refit-$v$ coefficients & App.~\ref{app:estimation} \\
$\rho,\ \rho_u$ & cold-start shrinkage; its decaying online schedule & Eq.~\eqref{eq:pricing-full} \\
$\nu$ & smoothing weight on the previous coefficients & Eq.~\eqref{eq:signgate} \\
$w_k$; $\bw^{(u)}$ & max-normalized price in $[0,1]$, $0$ unless $\beta_k<0$; prices of batch $u$ & Eq.~\eqref{eq:score} \\
\addlinespace[1pt]
\multicolumn{3}{@{}l}{\textbf{Score, budget and redistribution}}\\[1pt]
$S_{q,i}$ & rate-weighted, clipped error score in $[0,1]$ & Eq.~\eqref{eq:score} \\
$\lambda_u$; $\Tbar$ & penalty strength (warm-up, cosine decay); first-batch mean length & Eq.~\eqref{eq:score} \\
$P_{q,i}$ & budget $\min(\lambda_uS_{q,i},\eta\,\dR)$ & Eq.~\eqref{eq:score} \\
$h_{q,i,t}$ & total price located at step $t$ & Eq.~\eqref{eq:redistribution} \\
$\delta_{q,i,t}$; $\bar\delta$ & step share of the budget, $\sum_t\delta_{q,i,t}=P_{q,i}$; group-mean share & Eq.~\eqref{eq:redistribution} \\
\addlinespace[1pt]
\multicolumn{3}{@{}l}{\textbf{Dual-channel advantage and policy update}}\\[1pt]
$E_{q,i,t}$; $\Omega_q$ & episode score $\Rout_{q,i}-\delta_{q,i,t}$; all step samples of task $q$ & Eq.~\eqref{eq:aep} \\
$\Zt_{q,i,t}$; $\gamma$ & discounted return from step $t$ minus the share; discount factor & Eq.~\eqref{eq:astep} \\
$\calH(o)$ & anchor-state group: same-task samples with similar contexts & Eq.~\eqref{eq:astep} \\
$\Aep,\ \Astep,\ A_{q,i,t}$; $\omega$ & episode, step and joint advantage $\Aep+\omega\Astep$; step weight & Eq.~\eqref{eq:aep} \\
$\calU_{q,i,t}$; $\ell$; $N$ & tokens of response $t$; token index; batch token count & Eq.~\eqref{eq:actor-main} \\
$\varrho_\ell$; $\epsilon$ & ratio $\pi_\theta(a_\ell\mid c_\ell)/\piold(a_\ell\mid c_\ell)$; two-sided clip width & Eq.~\eqref{eq:actor-main} \\
$\DLV_\ell$, $\lambda_{\mathrm{KL}}$; $\calL_{\mathrm{actor}}$ & token KL estimate to $\piref$, its weight; actor loss & Eq.~\eqref{eq:actor-main} \\
\end{longtable}
\endgroup

\section{Method Details}
\label{app:method}
This appendix details each module of \S\ref{sec:method} in the order of the training pipeline, from the cold start (\ref{app:coldstart}) to the full algorithm and its implementation notes (\ref{app:algorithm}).
Appendix~\ref{app:exp} gives the benchmark-specific instantiation and Appendix~\ref{app:hparams} the hyper-parameters.

\subsection{Cold Start: Taxonomy, Diagnosis SFT and Initial Prices}
\label{app:coldstart}
Before RL, a one-time cold start produces a frozen error taxonomy, a diagnosis-SFT checkpoint and initial error prices.
Appendix~\ref{app:prompts} gives the diagnosis prompt; Appendix~\ref{app:exp} provides the data and training settings.

\paragraphhd{Teacher diagnosis.}
We collect training trajectories from the base policy and, optionally, early checkpoints trained with outcome-only RL.
A teacher diagnoses each trajectory once using greedy decoding, given its task, outcome, reward and step count.
Each error report identifies a mechanism through a free-form label, a brief explanation, cited steps, a verbatim quote and a severity hint.
The teacher is used only during cold start.

\paragraphhd{Taxonomy construction.}
An LLM organizes the accepted labels into L1 error families and L2 mechanisms.
Starting from an empty table, a growth pass processes labels in descending frequency, matching existing mechanisms or adding new classes when uncertain.
A second pass assigns the remaining labels to existing classes or marks them as off-table.
Rare L2 classes are merged back into their L1 families, and the taxonomy is frozen.
We select $K$ pricing classes $\calC$ from the retained L2 classes and L1 fallback categories using minimum verified counts and frequency-rank thresholds.
Selection ensures data support; the fitted price may still be zero.

\paragraphhd{Diagnosis SFT.}
Teacher diagnoses mapped to the frozen taxonomy become the structured targets $\bdiag_{q,i}$.
We reserve a held-out set stratified by checkpoint, outcome and task type; the remaining diagnoses form $\calD_{\mathrm{sft}}$, supplemented during training by a few action-replay examples from successful trajectories.
Each diagnosis input contains the rules, taxonomy, task, outcome and trajectory, with the teacher's JSON as its target.
We fine-tune all model parameters using Eq.~\eqref{eq:sft-main}; the diagnosis targets contain no prices, scores or advantages.

\paragraphhd{Initial prices.}
The verified teacher diagnoses also provide the error rates used with terminal outcomes to fit $\bhatv{0}$ and obtain $\bw^{(1)}$.
Appendix~\ref{app:price-init} details this fit, which learns prices separately from diagnosis SFT.

\subsection{Online Self-Diagnosis, Verification and Aggregation}
\label{app:diag}
Online diagnosis provides error rates for pricing and step evidence for penalty allocation.

\paragraphhd{Diagnosis generation.}
Before each actor update, the unchanged rollout policy samples $M$ diagnoses per completed trajectory at a fixed temperature, given the frozen taxonomy, task, outcome and trajectory.
Each diagnosis is JSON containing the reported outcome and a bounded list of entries $(k,t,e)$: class, step and quote.
Free-form labels outside the taxonomy are logged separately and excluded from scoring.

\paragraphhd{Entry verification.}
An entry is verified when its class belongs to the frozen taxonomy (L2, or an L1 fallback when L2 is absent), its step $t\in\{1,\ldots,T_{q,i}\}$ was executed, and its quote meets the minimum length and matches a verbatim, case-sensitive substring of that step's text (observation, response and environment feedback).
Sampling settings and verification thresholds are listed in Table~\ref{tab:hparams}.

\paragraphhd{Trajectory selection.}
A trajectory passes the gate when at least two diagnoses are parsable JSON and a majority of the parsable diagnoses correctly report its outcome.
If it fails, it is excluded from price fitting; its score $S_{q,i}$ is filled with the mean score of same-outcome trajectories in the batch and marked as imputed.

\paragraphhd{Error aggregation.}
We map verified entries to pricing classes.
Let $\Mval$ be the parsable diagnoses of $\tau_{q,i}$, $M_{q,i}=|\Mval|$, and $n_{q,i,m,k}$ the number of verified class-$k$ entries in diagnosis $m$, counting repeated entries separately.
We average these counts across all parsable diagnoses and normalize by the number of executed steps:
\begin{equation}
n_{q,i,k}=\frac{1}{M_{q,i}}\sum_{m\in\Mval}n_{q,i,m,k},
\qquad
x_{q,i,k}=\frac{n_{q,i,k}}{T_{q,i}}.
\label{eq:softcount}
\end{equation}
Averaging retains a verified error reported by only one diagnosis, but reduces its contribution.
We also record $k^{*}_{q,i}$, the class of the earliest verified error across all parsable diagnoses.
The rates form the pricing record $(q,\bx_{q,i},y_{q,i})$; for graded outcomes, the normalized reward replaces $y_{q,i}$ as defined in Appendix~\ref{app:estimation}.

\subsection{Terminal-Anchored Error Pricing}
\label{app:pricing}
We estimate how verified error rates relate to terminal outcomes while accounting for differences in task difficulty.
We first remove the task-specific intercept, derive the regularized estimators, and then describe how their coefficients produce initial prices and online updates.

\subsubsection{Conditional Likelihood}
Let $\mathcal{G}_q$ contain the trajectories retained for fitting in group $q$, with error-rate vectors $\bx_{q,i}$ and binary outcomes $y_{q,i}$.
The observed success set is $\calY_q=\{i\in\mathcal{G}_q:y_{q,i}=1\}$, with size $s_q=|\calY_q|$.
We use a logistic model with shared error coefficients $\bbeta$ and a task-specific intercept $\alpha_q$:
\begin{equation}
p_{q,i}=\Pr(y_{q,i}=1\mid\bx_{q,i},q)
=\sigmoid(\alpha_q+\bbeta^{\top}\bx_{q,i}).
\label{eq:pricing-logistic}
\end{equation}
Conditioning on $s_q$ removes $\alpha_q$ without estimating a separate difficulty parameter for every group~\citep{chamberlain1980analysis}.
To show this, define the candidate success sets and their feature sums as
\[
\calA_q=\{\Gamma\subseteq\mathcal{G}_q:|\Gamma|=s_q\},
\qquad
\bfeat_{q,\Gamma}=\sum_{i\in\Gamma}\bx_{q,i}.
\]
Assuming conditionally independent outcomes under the model, the probability that exactly the trajectories in $\Gamma$ succeed is
\begin{equation}
\begin{aligned}
\Pr(\Gamma\mid\alpha_q,\bbeta)
&=\prod_{i\in\Gamma}p_{q,i}\prod_{i\in\mathcal{G}_q\setminus\Gamma}(1-p_{q,i})\\
&=\frac{\exp\!\big(|\Gamma|\alpha_q+\bbeta^{\top}\bfeat_{q,\Gamma}\big)}
{\prod_{i\in\mathcal{G}_q}\big[1+\exp(\alpha_q+\bbeta^{\top}\bx_{q,i})\big]},
\end{aligned}
\label{eq:success-set-joint}
\end{equation}
where conditioning on the group's feature vectors is implicit.
The denominator is common to all candidate sets.
After conditioning on the observed success count, this denominator cancels, followed by the common factor $\exp(s_q\alpha_q)$:
\begin{equation}
\begin{aligned}
\Pr\!\left(\calY_q\,\middle|\,\sum_{i\in\mathcal{G}_q}y_{q,i}=s_q\right)
&=\frac{\Pr(\calY_q\mid\alpha_q,\bbeta)}
{\sum_{\Gamma\in\calA_q}\Pr(\Gamma\mid\alpha_q,\bbeta)}\\
&=\frac{\exp(s_q\alpha_q)\exp(\bbeta^{\top}\bfeat_{q,\calY_q})}
{\exp(s_q\alpha_q)\sum_{\Gamma\in\calA_q}\exp(\bbeta^{\top}\bfeat_{q,\Gamma})}\\
&=\frac{\exp(\bbeta^{\top}\bfeat_{q,\calY_q})}
{\sum_{\Gamma\in\calA_q}\exp(\bbeta^{\top}\bfeat_{q,\Gamma})}
=\Pq(\bbeta).
\end{aligned}
\label{eq:conditional-cancellation}
\end{equation}
This recovers Eq.~\eqref{eq:condlogit} using only within-group comparisons.
If all trajectories succeed or all fail, $\calA_q$ contains one set and $\Pq(\bbeta)=1$, so the group contributes neither loss nor gradient to this estimator.
The binary fit therefore uses $\calQ_{\mathrm{info}}$, the groups containing both outcomes.

\subsubsection{Regularized Coefficient Estimation}
\label{app:estimation}
\paragraphhd{Binary outcomes.}
Write $\ell(\bbeta)=\sum_{q\in\calQ_{\mathrm{info}}}\log\Pq(\bbeta)$.
For each candidate set, let $p_{q,\Gamma}(\bbeta)$ denote its conditional probability from Eq.~\eqref{eq:conditional-cancellation}, with $\Gamma$ in place of $\calY_q$.
Differentiating the log-sum-exp normalizer gives
\begin{equation}
\begin{aligned}
\boldsymbol\mu_q(\bbeta)
&=\sum_{\Gamma\in\calA_q}p_{q,\Gamma}(\bbeta)\bfeat_{q,\Gamma},\\
\nabla\ell(\bbeta)
&=\sum_{q\in\calQ_{\mathrm{info}}}\big(\bfeat_{q,\calY_q}-\boldsymbol\mu_q\big),\\
\calI(\bbeta)=-\nabla^2\ell(\bbeta)
&=\sum_{q\in\calQ_{\mathrm{info}}}\sum_{\Gamma\in\calA_q}p_{q,\Gamma}(\bbeta)
\big(\bfeat_{q,\Gamma}-\boldsymbol\mu_q\big)
\big(\bfeat_{q,\Gamma}-\boldsymbol\mu_q\big)^{\top}.
\end{aligned}
\label{eq:conditional-information}
\end{equation}
Thus the information matrix is the sum of the within-group covariances of candidate feature sums.
To stabilize estimation from limited data, we fit
\begin{equation}
\begin{aligned}
\bhat(\bref)
&=\argmin_{\|\bbeta\|_{\infty}\le\beta_{\max}}\calL_{\mathrm{price}}(\bbeta),\\
\calL_{\mathrm{price}}(\bbeta)
&=-\ell(\bbeta)-\tfrac12\logdet\calI(\bbeta)
+\rho\|\bbeta-\bref\|_2^2.
\end{aligned}
\label{eq:pricing-full}
\end{equation}
The first term fits within-task outcomes; the Firth correction addresses small-sample bias and separation~\citep{firth1993bias,heinze2002separation}; and the final term shrinks uncertain coefficients toward $\bref$.
The log-determinant requires positive-definite information on the fitted coordinates; regularization does not make indistinguishable error classes identifiable.
We enumerate candidate sets exactly and solve the bounded objective with L-BFGS-B~\citep{byrd1995lbfgsb,zhu1997algorithm778}.
The resulting $\bhat$ is the regularized numerical estimate used by the method, whereas $\bbeta^{\star}$ in Eq.~\eqref{eq:condlogit} denotes the unpenalized estimator.

\paragraphhd{Graded outcomes.}
For graded rewards, let $z_{q,i}=\clip((\Rout_{q,i}-\Rmin)/\dR,0,1)$.
We replace the binary likelihood by squared-error regression with unpenalized group intercepts:
\begin{equation}
\min_{\bbeta,\{\alpha_q\}}
\sum_{q,i}\big(z_{q,i}-\alpha_q-\bbeta^{\top}\bx_{q,i}\big)^2
+\rho\|\bbeta-\bref\|_2^2,
\label{eq:graded-intercepts}
\end{equation}
where the sum covers the retained trajectories in the fitting data.
For fixed $\bbeta$, setting the derivative with respect to $\alpha_q$ to zero gives
$\widehat\alpha_q=\bar z_q-\bbeta^{\top}\bar{\bx}_q$, where the bars denote means within group $q$.
Substituting this expression removes the intercepts by centring both features and rewards:
$\bx^c_{q,i}=\bx_{q,i}-\bar{\bx}_q$ and $z^c_{q,i}=z_{q,i}-\bar z_q$.
Let $\mathbf X_c$ contain the centred feature rows and $\mathbf z_c$ the centred rewards.
The reduced objective and its normal equations yield
\begin{equation}
\begin{aligned}
\calL_{\mathrm{ridge}}(\bbeta)
&=\|\mathbf z_c-\mathbf X_c\bbeta\|_2^2+\rho\|\bbeta-\bref\|_2^2,\\
(\mathbf X_c^{\top}\mathbf X_c+\rho\mathbf I)\bhat(\bref)
&=\mathbf X_c^{\top}\mathbf z_c+\rho\bref,\\
\bhat(\bref)
&=(\mathbf X_c^{\top}\mathbf X_c+\rho\mathbf I)^{-1}
(\mathbf X_c^{\top}\mathbf z_c+\rho\bref).
\end{aligned}
\label{eq:ridge-price}
\end{equation}
Here $\mathbf I$ is the identity matrix and $\rho>0$ makes the system invertible.
This ridge estimator replaces the likelihood and Firth terms for graded outcomes; the price mapping and update rules below are shared by both branches.

\subsubsection{From Coefficients to Prices}
Negative coefficients associate higher error rates with lower outcomes within a task.
We use their magnitudes as relative costs and assign no reward to positive coefficients:
\begin{equation}
\widetilde w_k=\max(-\beta_k,0)\,\sigma_k,
\qquad
w_k=\begin{cases}
\widetilde w_k/\max_j\widetilde w_j, & \max_j\widetilde w_j>0,\\[1pt]
0, & \text{otherwise.}
\end{cases}
\label{eq:prices}
\end{equation}
At cold start, $\sigma_k$ is the population standard deviation of error rate $k$ over all cold-start trajectories, including zeros and groups excluded from the binary likelihood.
Then $|\beta_k|\sigma_k$ measures the change in the linear predictor associated with one standard deviation of that rate.
For online refits, $\sigma_k=1$, recovering Eq.~\eqref{eq:score}.
When any price is positive, the most expensive class has price one; $\lambda_u$ controls the overall penalty strength (Appendix~\ref{app:budget}).
These prices summarize predictive associations with outcomes, not causal effects.

\subsubsection{Cold-Start Fit and Initialization}
\label{app:price-init}
Cold-start fitting uses the teacher-diagnosed trajectories of Appendix~\ref{app:coldstart}, with $M_{q,i}=1$ in Eq.~\eqref{eq:softcount}.
Each group $q$ consists of the rollouts of one task from one checkpoint.
To separate error density from the additional association between trajectory length and outcome, the initial fit includes a length covariate:
\begin{equation}
\widetilde{\bx}_{q,i}=
\begin{pmatrix}\bx_{q,i}\\\Tfit_{q,i}\end{pmatrix},
\qquad \Tfit_{q,i}=\frac{T_{q,i}}{T^{\max}},
\qquad \widetilde{\bbeta}=
\begin{pmatrix}\bbeta\\\beta_T\end{pmatrix}.
\label{eq:pricing-length}
\end{equation}
Unlike the group intercept, this covariate can vary within a group and is therefore retained after intercept removal.
We apply the appropriate estimator of Appendix~\ref{app:estimation} to these extended features, using a zero reference and fixed shrinkage strength $\rho$; the binary optimizer also starts at zero and bounds the length coefficient.
The coefficient $\beta_T$ controls for length in fitting but never enters the price vector.

Let $\bhatv{0}$ contain only the $K$ error coefficients from this fit.
Equation~\eqref{eq:prices}, with the cold-start scales, gives the initial prices $\bw^{(1)}$.
We also retain the non-positive part of the fitted error coefficients as a fixed reference for online estimation:
\begin{equation}
\bzeroinit=\min\big(\bhatv{0},\bzero\big),
\label{eq:price-anchor}
\end{equation}
where the minimum is coordinate-wise.
Thus $\bhatv{0}$ is the initial fit, $\bzeroinit$ is the frozen reference, and $\bw^{(1)}$ is the initial price vector.
Online refits use only the $K$ error rates, without the length covariate.

\subsubsection{Online Refits and Price Updates}
\label{app:online}
\paragraphhd{Fitting recent trajectories.}
Batch $u$ uses the fixed prices $\bw^{(u)}$.
After its actor update, the verified records $(q,\bx_{q,i},y_{q,i})$ with non-imputed scores enter a sliding window of the last $W$ updates; graded outcomes use $z_{q,i}$ in place of $y_{q,i}$.
Refits begin at $u_0$ and recur every $\Delta u$ updates.
The binary fit requires enough mixed-outcome groups and successes; only classes with sufficient non-zero observations and within-group variation are fitted (thresholds in Table~\ref{tab:hparams}).
If these checks fail or a fit is rejected, the previous prices remain in use.
Numerical binary fits require convergence, a finite objective, and finite, feasible coefficients; implementation differences are recorded in Appendix~\ref{app:algorithm}.

For the $v$-th accepted refit, we use the current window to obtain $\bhatv{v}=\bhat(\bzeroinit)$ with shrinkage strength $\rho_u$, which decreases over training.
The reference stays fixed at $\bzeroinit$; each binary optimization also starts there.
The previous online coefficients enter only the subsequent smoothing step, not the shrinkage target.

\paragraphhd{Updating the deployed coefficients.}
Let $\calK_v$ be the classes included in the accepted refit.
Starting from $\bzeroinit$, update the deployed coefficients by
\begin{equation}
\beta^{(v)}_k=
\begin{cases}
\nu\beta^{(v-1)}_k+(1-\nu)\hat\beta^{(v)}_k,
& k\in\calK_v\ \text{and}\ \hat\beta^{(v)}_k\le0,\\[2pt]
\beta^{(v-1)}_k, & \text{otherwise,}
\end{cases}
\qquad \nu=0.7.
\label{eq:signgate}
\end{equation}
This exponential moving average gives weight $0.7$ to the previous value and $0.3$ to the new fit.
An unfitted class or a positive new coefficient keeps its previous value.
We map $\bv$ to prices with $\sigma_k=1$ in Eq.~\eqref{eq:prices}; the new prices take effect from batch $u+1$, so a batch is never scored using a fit to its own outcomes.
The smoothing acts on coefficients; the change in scaling after cold start means that it does not guarantee continuity of the normalized prices.

\subsection{Score, Penalty Schedule and Order-Preserving Budget}
\label{app:budget}
With error prices fixed within a batch, we compute a trajectory score, scale it over training, and cap the resulting penalty before distributing it across steps.
The cap preserves the ordering of binary outcomes in the episode channel, as shown below.

\subsubsection{Trajectory Score}
For a trajectory that passes the diagnosis gate, the batch prices $\bw=\bw^{(u)}$ and soft counts from Eq.~\eqref{eq:softcount} give
\begin{equation}
\begin{aligned}
\Sraw_{q,i}
&=\sum_{k=1}^{K}w_kn_{q,i,k}+(b-1)w_{k^{*}_{q,i}},\\
S_{q,i}
&=\clip\!\left(\frac{\Sraw_{q,i}}{T_{q,i}},0,1\right)\\
&=\clip\!\left(\sum_{k=1}^{K}w_kx_{q,i,k}
+\frac{(b-1)w_{k^{*}_{q,i}}}{T_{q,i}},0,1\right).
\end{aligned}
\label{eq:score-full}
\end{equation}
Here $b\ge1$ adds weight to the earliest verified error class $k^{*}_{q,i}$; the extra term is zero when no verified error exists.
We use a first-error penalty factor of $b=1.5$ in all experiments, adding $0.5\,w_{k^{*}_{q,i}}$ to $\Sraw_{q,i}$.
Setting $b=1$ recovers the score in Eq.~\eqref{eq:score}.
The rate term is independent of length at fixed error rates; the first-error term retains an explicit $1/T_{q,i}$ dependence.
Both successful and failed trajectories are scored; trajectories that fail the diagnosis gate use the imputed score of Appendix~\ref{app:diag}.

\subsubsection{Penalty Schedule}
Let $u=1,\ldots,U$ index actor updates and $0\le U_{\mathrm w}<U$ be the warm-up length.
We measure $\Tbar$, the mean executed trajectory length in the first rollout batch, before any actor update.
The non-negative strengths $\lambda_{\mathrm{warm}}$ and $\lambda_{\mathrm{peak}}$ are fixed multiples of $\Tbar$ (Table~\ref{tab:hparams}; benchmark variants in Appendix~\ref{app:inst-other}).
This scales a per-step error score to a trajectory-level penalty using the initial trajectory length.
We keep $\Tbar$ fixed so that later changes in trajectory length do not change this global scale.
The penalty strength is
\begin{equation}
\lambda_u=
\begin{cases}
\lambda_{\mathrm{warm}}, & u\le U_{\mathrm w},\\[3pt]
\dfrac{\lambda_{\mathrm{peak}}}{2}\Big[1+\cos\!\Big(\pi\,\dfrac{u-U_{\mathrm w}}{U-U_{\mathrm w}}\Big)\Big], & u>U_{\mathrm w}.
\end{cases}
\label{eq:lambda}
\end{equation}
The warm-up strength is constant; the remaining updates follow the cosine branch, with $\lambda_U=0$ so that the diagnostic penalty vanishes at the final update.

\subsubsection{Bounded Budget and Order Preservation}
Let $\dR=\Rmax-\Rmin>0$ be the terminal reward range, equal to the success--failure reward gap for binary outcomes.
For $0\le\eta<1$, the trajectory budget is
\begin{equation}
P_{q,i}=\min\big(\lambda_uS_{q,i},\eta\dR\big),
\qquad 0\le P_{q,i}\le\eta\dR.
\label{eq:cap}
\end{equation}
This budget is nondecreasing in $S_{q,i}$, with ties once the cap is reached; Appendix~\ref{app:redis} distributes it across steps without changing its total.

\begin{proposition}[Episode-channel order preservation]
\label{prop:order}
Fix a task $q$ with binary terminal rewards $\Rmin+\dR\,y_{q,i}$ and episode scores $E_{q,i,t}=\Rout_{q,i}-\delta_{q,i,t}$.
For any successful trajectory $i$, failed trajectory $j$, and their respective steps $t$ and $r$,
\begin{equation*}
\Aep_{q,i,t}-\Aep_{q,j,r}
=E_{q,i,t}-E_{q,j,r}
\ge(1-\eta)\dR>0.
\end{equation*}
\end{proposition}
\begin{proof}
The conserved split has non-negative shares, so $0\le\delta_{q,i,t}\le P_{q,i}$.
Using Eq.~\eqref{eq:cap},
\begin{equation*}
\begin{aligned}
E_{q,i,t}-E_{q,j,r}
&=\dR-\delta_{q,i,t}+\delta_{q,j,r}\\
&\ge\dR-P_{q,i}\\
&\ge(1-\eta)\dR>0.
\end{aligned}
\end{equation*}
Both episode advantages subtract the same task-group mean in Eq.~\eqref{eq:aep}, leaving their difference unchanged.
\end{proof}

The proposition concerns the diagnostic episode channel; additional reward terms in Appendix~\ref{app:adv} are outside its scope.
For graded outcomes, the same argument gives
\[
E_{q,i,t}-E_{q,j,r}
\ge\Rout_{q,i}-\Rout_{q,j}-\eta\dR.
\]
This bound guarantees strict order when the original reward gap exceeds the penalty cap.

\subsection{Conserved Step Redistribution}
\label{app:redis}
Given the bounded budget $P_{q,i}$ from Appendix~\ref{app:budget}, we use the diagnosed error locations to determine each step's share without changing the total penalty.

\subsubsection{Localization Weights}
Let $\calV_{q,i,t}$ collect the verified entries citing step $t$ across all parsable diagnoses, retaining repeated reports as separate entries.
Each entry contributes its class price, giving
\[
h_{q,i,t}=\sum_{e\in\calV_{q,i,t}}w_{k(e)}\ge0.
\]
Thus a class reported at the same step by three diagnoses contributes $3w_k$; the relative support across steps determines their allocation weights.
Quote length and severity hints do not affect these weights.
When at least one weight is positive, we boost the earliest such step:
\begin{equation}
t^{*}_{q,i}=\min\{t:h_{q,i,t}>0\},
\qquad
\widetilde h_{q,i,t}=\big[1+(b_{\mathrm{loc}}-1)\ind{t=t^{*}_{q,i}}\big]h_{q,i,t}.
\label{eq:hmap}
\end{equation}
If all weights are zero, we set $\widetilde h_{q,i,t}=0$ for every step.
We use $b_{\mathrm{loc}}=1.5$ (Table~\ref{tab:hparams}).
Unlike the score-level boost $b$ in Appendix~\ref{app:budget}, which can increase the total budget, $b_{\mathrm{loc}}$ changes only its allocation.
The boosted step is the earliest with positive price, which need not be the earliest verified error.

\subsubsection{Normalized Allocation and Conservation}
The complete allocation rule is
\begin{equation}
\delta_{q,i,t}=
\begin{cases}
\displaystyle P_{q,i}\frac{\widetilde h_{q,i,t}}{\sum_{r=1}^{T_{q,i}}\widetilde h_{q,i,r}},
& \displaystyle\sum_{r=1}^{T_{q,i}}\widetilde h_{q,i,r}>0,\\[9pt]
P_{q,i}\ind{t=T_{q,i}}, & \text{otherwise}.
\end{cases}
\label{eq:redistribution-full}
\end{equation}
This is Eq.~\eqref{eq:redistribution} with the boosted weights and a fallback for missing locations.
A positive budget without a positive localization weight can arise from score imputation (Appendix~\ref{app:diag}); it is assigned to the last active step.
Both branches use non-negative proportions summing to one, so
\begin{equation}
0\le\delta_{q,i,t}\le P_{q,i},
\qquad
\sum_{t=1}^{T_{q,i}}\delta_{q,i,t}=P_{q,i}.
\label{eq:redistribution-conservation}
\end{equation}
With localized evidence, zero-weight steps receive no penalty; when $P_{q,i}=0$, all shares vanish.
Redistribution therefore changes where the budget is charged, while preserving the total set by the rate-based score, schedule and cap.

\subsection{Dual-Channel Advantage and Policy Update}
\label{app:adv}
The step shares $\delta_{q,i,t}$ enter two comparisons: an episode channel across the task's rollouts and a step channel across similar pre-response contexts.
Their combined advantage then supervises the response tokens of each step.

\subsubsection{Comparison Groups and Local Scores}
The task group contains all active step samples,
\[
\Omega_q=\{(j,r):1\le j\le G,\ 1\le r\le T_{q,j}\},
\qquad |\Omega_q|=\sum_{j=1}^{G}T_{q,j}.
\]
Each step counts once, so longer trajectories contribute more samples to the episode baseline.
For the step channel, $\calH(o_{q,i,t})\subseteq\Omega_q$ contains samples clustered with $o_{q,i,t}$ by pre-response context similarity.
A context joins a cluster when its text similarity to the representative meets the threshold in Table~\ref{tab:hparams}; otherwise it starts a new cluster.
These groups approximate comparable contexts rather than asserting identical latent states.

We first compute the return from the environment rewards, then subtract the diagnostic share locally:
\begin{equation}
\begin{aligned}
\Zenv_{q,i,t}&=\sum_{r=t}^{T_{q,i}}\gamma^{r-t}R^{\mathrm{env}}_{q,i,r},\\
E_{q,i,t}&=\Rout_{q,i}-\delta_{q,i,t},
\qquad
\Zt_{q,i,t}=\Zenv_{q,i,t}-\delta_{q,i,t}.
\end{aligned}
\label{eq:adv-local-scores}
\end{equation}
For purely terminal rewards, $\Zenv_{q,i,t}=\gamma^{T_{q,i}-t}\Rout_{q,i}$, recovering Eq.~\eqref{eq:astep}.
The share is not added to the environment reward and discounted back to earlier steps: the diagnosis has already identified where to charge it.
When the base recipe includes a rule-based invalid-action penalty, it also enters $E_{q,i,t}$; the displayed equations omit this additional term.

\subsubsection{Centred Advantages and the Effect of Diagnosis}
Following the two comparison groups, we define
\begin{equation}
\begin{aligned}
\muep_q&=\frac{1}{|\Omega_q|}\sum_{(j,r)\in\Omega_q}E_{q,j,r},
&\Aep_{q,i,t}&=E_{q,i,t}-\muep_q,\\
\mustep(o)&=\frac{1}{|\calH(o)|}\sum_{(j,r)\in\calH(o)}\Zt_{q,j,r},
&\Astep_{q,i,t}&=\Zt_{q,i,t}-\mustep(o_{q,i,t}),\\
&&A_{q,i,t}&=\Aep_{q,i,t}+\omega\Astep_{q,i,t}.
\end{aligned}
\label{eq:adv-dual-full}
\end{equation}
On ALFWorld and WebShop, both channels subtract a mean without dividing by a standard deviation, keeping advantages in reward units; Search-based QA also divides by the standard deviation (Table~\ref{tab:hparams}).
A singleton anchor group has $\Astep=0$, while its episode advantage can remain non-zero.

To isolate the effect of diagnosis, hold the rollouts, comparison groups and other reward terms fixed, and let $A^{(0)}_{q,i,t}$ be the joint advantage with $\delta=0$.
For a non-empty comparison set $\mathcal G\subseteq\Omega_q$, write $\bar\delta_{\mathcal G}=|\mathcal G|^{-1}\sum_{(j,r)\in\mathcal G}\delta_{q,j,r}$.
Subtracting the two versions of Eq.~\eqref{eq:adv-dual-full} gives
\begin{equation}
A_{q,i,t}-A^{(0)}_{q,i,t}
=-\big(\delta_{q,i,t}-\bar\delta_{\Omega_q}\big)
-\omega\big(\delta_{q,i,t}-\bar\delta_{\calH(o_{q,i,t})}\big).
\label{eq:adv-diagnostic-shift}
\end{equation}
Thus each channel lowers a step's advantage relative to its own comparison group when its penalty exceeds that group's mean; a uniform penalty within a group cancels.
The two channels compare the same local cost against different reference sets, rather than assigning two separate budgets.

\subsubsection{Token-Level Policy Update}
For each response-token position $\ell\in\calU_{q,i,t}$, we use the same joint advantage, $A_\ell=A_{q,i,t}$, in Eq.~\eqref{eq:actor-main}.
The loss averages over the $N=\sum_{(q,i,t)}|\calU_{q,i,t}|$ response tokens in the batch; context, feedback and diagnosis tokens receive no policy loss.
The importance ratio is $\varrho_\ell=\pi_\theta(a_\ell\mid c_\ell)/\piold(a_\ell\mid c_\ell)$; both $\piold$ and the advantages are held fixed during the update.
In the experiments, the clipped surrogate additionally uses dual clipping for negative advantages~\citep{ye2020mastering}, and the loss includes the low-variance KL term $\lambda_{\mathrm{KL}}\DLV_\ell$ relative to the frozen policy $\piref$~\citep{schulman2020klapprox}.
The clipping and KL settings are listed in Table~\ref{tab:hparams}.

\subsection{Algorithm}
\label{app:algorithm}
Algorithm~\ref{alg:dtcr} organizes the RL stage into four steps: rollout and diagnosis, penalty allocation, policy update, and online price refitting.
Its initial policy, taxonomy and prices come from the cold start in Appendices~\ref{app:coldstart} and~\ref{app:price-init}; hyper-parameters are listed in Table~\ref{tab:hparams}.

\begin{algorithm}[h]
\caption{\method{}: self-evolving RL with diagnosis-guided credit}
\label{alg:dtcr}
\small
\begin{algorithmic}[1]
\Require Diagnosis-SFT policy $\pi_\theta$; frozen reference policy $\piref$; task distribution $\calD$
\Require Frozen taxonomy and pricing classes $\calC$; initial fit $\bhatv{0}$ and prices $\bw^{(1)}$
\Require Reward bounds $\Rmin,\Rmax$; training settings of Table~\ref{tab:hparams}
\Ensure Trained policy $\pi_\theta$ and price sequence $\{\bw^{(u)}\}_{u=1}^{U+1}$
\State $\bzeroinit\gets\min(\bhatv{0},\bzero)$; $v\gets0$; $\calW\gets\emptyset$; $\dR\gets\Rmax-\Rmin$
\For{$u=1,\ldots,U$}
  \State $\piold\gets\pi_\theta$; $\bw\gets\bw^{(u)}$ \Comment{fixed throughout the current update}
  \Statex \hspace{\algorithmicindent}\textit{1. Rollout and self-diagnosis}
  \State Sample tasks from $\calD$ and $G$ trajectories per task using $\piold$; record outcomes and rewards
  \State Generate $M$ diagnoses per trajectory with $\piold$; verify entries and apply the gate (App.~\ref{app:diag})
  \State Retain verified entries mapped to pricing classes in $\calV_{q,i,t}$ for all trajectories
  \State For gate-passing trajectories, compute $n_{q,i,k},\bx_{q,i},k^{*}_{q,i}$ and $S_{q,i}$ by Eqs.~\eqref{eq:softcount}, \eqref{eq:score-full}
  \State Impute the remaining scores from the batch's same-outcome, non-imputed scores (App.~\ref{app:diag})
  \Statex \hspace{\algorithmicindent}\textit{2. Bounded penalty allocation}
  \If{$u=1$}
    \State Measure $\Tbar$; fix $\lambda_{\mathrm{warm}}$ and $\lambda_{\mathrm{peak}}$ as specified in App.~\ref{app:budget}
  \EndIf
  \State Compute $\lambda_u$ by Eq.~\eqref{eq:lambda}; set $P_{q,i}\gets\min(\lambda_uS_{q,i},\eta\dR)$ for every trajectory
  \State Compute $h_{q,i,t}\gets\sum_{e\in\calV_{q,i,t}}w_{k(e)}$ and its boosted form $\widetilde h_{q,i,t}$ by Eq.~\eqref{eq:hmap}
  \State Allocate $P_{q,i}$ using Eq.~\eqref{eq:redistribution-full}, including its zero-weight fallback, to obtain $\delta_{q,i,t}$
  \Statex \hspace{\algorithmicindent}\textit{3. Dual-channel policy update}
  \State Form the task groups $\Omega_q$ and context-based anchor groups $\calH(o_{q,i,t})$ (App.~\ref{app:adv})
  \State Compute $\Zenv_{q,i,t}$ from environment rewards, then $E_{q,i,t}$ and $\Zt_{q,i,t}$ by Eq.~\eqref{eq:adv-local-scores}
  \State Centre both channels and combine $A_{q,i,t}\gets\Aep_{q,i,t}+\omega\Astep_{q,i,t}$ by Eq.~\eqref{eq:adv-dual-full}
  \State Broadcast $A_{q,i,t}$ to its response tokens and update $\theta$ using Eq.~\eqref{eq:actor-main} (App.~\ref{app:adv})
  \Statex \hspace{\algorithmicindent}\textit{4. Price refitting for the next batch}
  \State Append non-imputed records $(q,\bx_{q,i},y_{q,i})$ to $\calW$; use $z_{q,i}$ for graded outcomes (App.~\ref{app:estimation})
  \State Preserve rollout-group identities in $\calW$ and retain only the last $W$ updates
  \State $\bw^{(u+1)}\gets\bw^{(u)}$ \Comment{retain prices unless this refit is accepted}
  \If{$u\ge u_0$, $(u-u_0)\bmod\Delta u=0$, and the refit checks of App.~\ref{app:online} hold}
    \State Set $\bref\gets\bzeroinit$ and $\rho\gets\rho_u$; fit $\bhat(\bzeroinit)$ on the supported classes in $\calW$
    \Statex \hspace{\algorithmicindent}\hspace{\algorithmicindent}Binary: Eq.~\eqref{eq:pricing-full}, initialized at $\bzeroinit$; graded: Eq.~\eqref{eq:ridge-price}
    \If{the current fit is accepted (App.~\ref{app:online})}
      \State $v\gets v+1$; $\bhatv{v}\gets\bhat(\bzeroinit)$; let $\calK_v$ be the fitted classes
      \State Update $\bv$ from $\bvm$ and $\bhatv{v}$ by the sign-gated EMA in Eq.~\eqref{eq:signgate}
      \State Set $\bw^{(u+1)}$ from $\bv$ using Eq.~\eqref{eq:prices} with $\sigma_k=1$
    \EndIf
  \EndIf
\EndFor
\end{algorithmic}
\end{algorithm}

\paragraphhd{Implementation notes.}
The reported runs have the following differences from the algorithmic specification.
With exactly two parsable diagnoses, the trajectory gate accepts a one-to-one tie; a finite fit with a false solver-success flag is logged but not rejected.
The schedule strengths are read from the configuration after being set from the measured $\Tbar$.
Earlier development runs only warned when the penalty cap was exceeded; the specified algorithm enforces it.

\section{Additional Experimental Setup}
\label{app:exp}
This appendix expands the benchmarks, baselines, evaluation protocols and implementation of \S\ref{sec:exp-setting}.

\subsection{Benchmarks and Baselines}
\paragraphhd{Benchmarks.}
ALFWorld comprises six families of household tasks (Pick, Look, Clean, Heat, Cool and Pick2); we train on its standard training games and evaluate on the unseen split.
WebShop tests product search, inspection and purchase against natural-language requests.
Search-based QA follows Search-R1~\citep{jin2025searchr1} and draws questions from NQ~\citep{kwiatkowski2019nq}, TriviaQA~\citep{joshi2017triviaqa}, PopQA~\citep{mallen2023popqa}, HotpotQA~\citep{yang2018hotpotqa}, 2WikiMultiHopQA~\citep{ho2020twowiki}, MuSiQue~\citep{trivedi2022musique} and Bamboogle~\citep{press2023bamboogle}.

\paragraphhd{Baselines.}
Vanilla uses the instruction-tuned backbone without task-specific post-training.
GRPO and GiGPO start from that backbone; GiGPO follows its upstream recipe, including $150$ updates on ALFWorld.
SEED uses its own teacher-written hindsight-skill SFT data and skill-distillation objective.
\method{} starts from diagnosis SFT, and the Qwen3-4B-Instruct-2507 GRPO (diag-SFT) and GiGPO (diag-SFT) controls share this checkpoint with diagnostic shaping disabled but diagnosis monitoring retained.

\subsection{Evaluation Protocols}
All reported metrics are percentages, with higher values better; each result comes from one training run, so differences are point estimates rather than claims of statistical significance.

\paragraphhd{ALFWorld.}
The unseen split has $134$ tasks: Pick $24$, Look $18$, Clean $31$, Heat $23$, Cool $21$ and Pick2 $17$.
Each task is evaluated three times by a single evaluation worker, with greedy decoding (temperature $0$) and at most $30$ active steps per evaluation.
We average the three success indicators for each task, then average across all $134$ tasks, so task families contribute in proportion to their size.

\paragraphhd{WebShop.}
We use the small catalogue of $1{,}000$ products and $500$ held-out goals.
Evaluation runs $8$ batches of $32$ goals ($256$ episodes) with greedy decoding, at most $15$ steps and a two-step history window.
Goals are sampled without replacement within a batch but may recur across batches, so the strict success rate counts successful episodes, not distinct goals.
The task score is the environment's graded score, averaged over episodes and scaled to $[0,100]$; the strict success rate is the fraction of episodes ending in a purchase of score one.

\paragraphhd{Search-based QA.}
The agent gathers evidence with a search tool before answering questions from the seven datasets above.
We report exact-match accuracy for each dataset and their unweighted mean, rather than pooling all questions into one accuracy.

\subsection{Cold-Start Data and Initialization}
\label{app:training-details}
Each benchmark has its own frozen error taxonomy, diagnosis-SFT data and initial price fit, constructed as in Appendix~\ref{app:coldstart}.
We describe the data and benchmark-specific choices below; training hyper-parameters are collected in Appendix~\ref{app:hparams}.

\subsubsection{ALFWorld}
\label{app:inst}
The following cold-start setup describes Qwen3-4B-Instruct-2507.

\paragraphhd{Data and teacher diagnosis.}
We sample eight rollouts for each of $180$ training tasks from two actor checkpoints: the base model and a checkpoint after $20$ GRPO updates.
Each checkpoint contributes $1{,}440$ trajectories, giving $2{,}880$ in total.
GPT-5.6 Sol diagnoses them at temperature $0$.

\paragraphhd{Taxonomy construction.}
Using GPT-5.6 Sol for both induction passes, we construct a frozen taxonomy with eight L1 families and $25$ L2 mechanisms.
From the L2 mechanisms and eight L1 fallback categories, we retain classes with at least ten verified occurrences and frequency ranks in the top $60\%$ of candidates, yielding $K=20$ pricing classes.

\paragraphhd{Diagnosis-SFT data.}
We reserve $200$ trajectories, stratified by checkpoint, outcome and task type, for assessment.
The remaining $2{,}680$ examples, supplemented by a few action-replay examples from successful trajectories, train the diagnosis-SFT model with the settings in Table~\ref{tab:hparams}.

\paragraphhd{Initial price fit.}
The fit uses $99$ informative groups of eight rollouts from the same task and checkpoint, or $792$ trajectories, with the extended features of Appendix~\ref{app:price-init}.
Cold-start feature scales $\sigma_k$ are computed over all $2{,}880$ trajectories.
The resulting prices are reported in Table~\ref{tab:coldprices}.

\subsubsection{WebShop}
The recorded Qwen3-4B-Instruct-2507 cold-start collection samples $180$ tasks and eight rollouts per task from each selected early checkpoint, producing up to $4{,}320$ trajectories.
WebShop uses its own error taxonomy and diagnosis-SFT data.
Its graded training reward uses the within-group ridge estimator of Appendix~\ref{app:estimation}.
The two model sizes use different penalty strengths, listed in Table~\ref{tab:hparams}.

\subsubsection{Search-Based QA}
Search-based QA also uses a benchmark-specific taxonomy and diagnosis-SFT data.
Its SFT data combine diagnosis examples with task responses, including successful solving trajectories from Qwen3.8-27B~\citep{qwen2026qwen3827b} that contain at least one search.
The 1.7B SFT set contains $6{,}974$ examples: $4{,}179$ task-response examples and $2{,}795$ diagnosis examples.
The 4B set contains $6{,}498$ examples: $3{,}699$ task-response examples and $2{,}799$ diagnosis examples.

\subsection{Training Configuration and Hyper-parameters}
\label{app:hparams}
\label{app:inst-other}
Table~\ref{tab:hparams} lists the six \method{} configurations in Table~\ref{tab:main}, separating shared settings from benchmark and model-size differences.
The backbones are Qwen3-1.7B and Qwen3-4B-Instruct-2507~\citep{yang2025qwen3,qwen2025qwen34binstruct2507}; both use full-parameter diagnosis SFT, and the actor and diagnoser share weights throughout RL.

\begin{figure}[!htbp]
\prompttitle{Diagnoser system prompt \ (\texttt{DIAG\_SYSTEM})}
\begin{lstlisting}[style=promptstyle]
You are the trajectory diagnostician for an AI agent. You will read ONE complete
trajectory of the agent solving a task in {env_name}, together with the FINAL OUTCOME,
and identify the ERRORS the agent made, (*\pk{using ONLY the error table below}*).

Error table (use these ids only):
{table_text}

Rules:
1. Report at most {max_errors} errors, most important first. One entry per distinct
   error mechanism; a cascade after the first error counts once.
2. Every entry MUST cite the step number ((*\pk{1-based}*), exactly as shown as (*\pk{[Step k]}*)) and a
   (*\pk{VERBATIM quote}*) of at most {max_quote_words} words copied character-for-character from
   that step's text (the agent's own text or the environment feedback shown under that
   step). Never paraphrase, never quote from another step.
3. Use l1/l2 ids from the table. If you see a real error that fits no row, put it in
   (*\pid{"off\_table"}*) (label, step, quote, one-line mechanism) instead of forcing an id.
4. A SUCCESSFUL trajectory can still contain errors (e.g., wasted retries). A FAILED
   trajectory may have no identifiable table error other than running out of steps:
   (*\pk{do NOT invent errors}*).
5. severity_hint is your judgement (minor|major|fatal); it does not change the outcome
   you were told.
6. (*\pk{Output ONLY one JSON object}*), no prose, exactly this schema:
{schema_text}
\end{lstlisting}

\prompttitle{Output schema \ (\texttt{DIAG\_SCHEMA\_TEXT}, injected as \texttt{\{schema\_text\}})}
\begin{lstlisting}[style=promptstyle]
{"outcome_seen": "success" | "failure",
 (*\pid{"errors"}*): [{"l1": "E1", "l2": "E1.2", "step": 7, "quote": "...", "severity_hint": "major"}],
 (*\pid{"dominant"}*): {"l1": "E1", "step": 7} | null,
 (*\pid{"off\_table"}*): [{"label": "...", "step": 12, "quote": "...", "mechanism": "..."}]}
\end{lstlisting}
\caption{\textbf{The frozen-table diagnosis prompt.}
Each reported error needs a frozen-table id (or \texttt{off\_table}), a $1$-based step and a verbatim quote.
The user turn gives the task, a one-line environment description, the final outcome with its reward and the rendered trajectory; only the environment name and description and the table text depend on the benchmark.}
\label{fig:prompt-diag}
\end{figure}

\begingroup
\footnotesize
\setlength{\tabcolsep}{3pt}
\setlength{\LTpre}{6pt}
\setlength{\LTpost}{6pt}
\renewcommand{\arraystretch}{1.08}
\begin{longtable}{@{}>{\raggedright\arraybackslash}p{4.0cm}*{6}{>{\centering\arraybackslash}p{1.4cm}}@{}}
\caption{Key training and method hyper-parameters for the six \method{} runs. Shared values span model or benchmark columns.}
\label{tab:hparams}\\
\toprule
Setting & \multicolumn{2}{c}{ALFWorld} & \multicolumn{2}{c}{WebShop} & \multicolumn{2}{c}{Search-based QA}\\
\cmidrule(lr){2-3}\cmidrule(lr){4-5}\cmidrule(l){6-7}
 & 1.7B & 4B & 1.7B & 4B & 1.7B & 4B\\
\midrule
\endfirsthead
\multicolumn{7}{@{}l}{Table~\thetable\ (continued)}\\
\toprule
Setting & \multicolumn{2}{c}{ALFWorld} & \multicolumn{2}{c}{WebShop} & \multicolumn{2}{c}{Search-based QA}\\
\cmidrule(lr){2-3}\cmidrule(lr){4-5}\cmidrule(l){6-7}
 & 1.7B & 4B & 1.7B & 4B & 1.7B & 4B\\
\midrule
\endhead
\midrule
\multicolumn{7}{r@{}}{Continued on next page}\\
\endfoot
\bottomrule
\endlastfoot
\multicolumn{7}{@{}l}{\textit{Supervised initialization}}\\*
\midrule
SFT learning rate & \multicolumn{6}{>{\raggedright\arraybackslash}p{9.45cm}@{}}{$10^{-5}$}\\
SFT learning-rate schedule & \multicolumn{6}{>{\raggedright\arraybackslash}p{9.45cm}@{}}{$10\%$ warm-up; cosine decay}\\
SFT epochs / global batch & \multicolumn{2}{>{\centering\arraybackslash}p{3.01cm}}{$2$ / $32$} & \multicolumn{2}{>{\centering\arraybackslash}p{3.01cm}}{$7$ / $176$} & \multicolumn{2}{>{\centering\arraybackslash}p{3.01cm}@{}}{$2$ / $32$}\\
\midrule
\multicolumn{7}{@{}l}{\textit{Policy optimization and rollout generation}}\\*
\midrule
RL learning rate / updates & \multicolumn{6}{>{\raggedright\arraybackslash}p{9.45cm}@{}}{$10^{-6}$; $U=160$; one PPO epoch per update}\\
Policy clipping & \multicolumn{6}{>{\raggedright\arraybackslash}p{9.45cm}@{}}{Lower/upper clipping $0.2/0.2$; dual-clip parameter $3.0$}\\
Rollout sampling & \multicolumn{6}{>{\raggedright\arraybackslash}p{9.45cm}@{}}{$G=8$ trajectories per task; temperature $1.0$, top-$p=1$, no top-$k$ filtering; thinking disabled}\\
Response limit (tokens) & \multicolumn{6}{>{\raggedright\arraybackslash}p{9.45cm}@{}}{$512$}\\
Tasks / trajectories per update & \multicolumn{2}{>{\centering\arraybackslash}p{3.01cm}}{$16$ / $128$} & \multicolumn{2}{>{\centering\arraybackslash}p{3.01cm}}{$16$ / $128$} & \multicolumn{2}{>{\centering\arraybackslash}p{3.01cm}@{}}{$128$ / $1{,}024$}\\
PPO mini-batch & \multicolumn{2}{>{\centering\arraybackslash}p{3.01cm}}{$256$} & \multicolumn{2}{>{\centering\arraybackslash}p{3.01cm}}{$256$} & \multicolumn{2}{>{\centering\arraybackslash}p{3.01cm}@{}}{$512$}\\
RL learning-rate warm-up fraction & \multicolumn{2}{>{\centering\arraybackslash}p{3.01cm}}{$0$} & \multicolumn{2}{>{\centering\arraybackslash}p{3.01cm}}{$0$} & \multicolumn{2}{>{\centering\arraybackslash}p{3.01cm}@{}}{$0.1$}\\
KL coefficient & \multicolumn{2}{>{\centering\arraybackslash}p{3.01cm}}{$0.01$} & \multicolumn{2}{>{\centering\arraybackslash}p{3.01cm}}{$0.01$} & \multicolumn{2}{>{\centering\arraybackslash}p{3.01cm}@{}}{$0.001$}\\
Entropy coefficient & \multicolumn{2}{>{\centering\arraybackslash}p{3.01cm}}{$0.001$} & \multicolumn{2}{>{\centering\arraybackslash}p{3.01cm}}{$0.001$} & \multicolumn{2}{>{\centering\arraybackslash}p{3.01cm}@{}}{$0$}\\
Invalid-action penalty & \multicolumn{2}{>{\centering\arraybackslash}p{3.01cm}}{$0.1$} & \multicolumn{2}{>{\centering\arraybackslash}p{3.01cm}}{$0.1$} & \multicolumn{2}{>{\centering\arraybackslash}p{3.01cm}@{}}{$0.01$}\\
Environment reward & \multicolumn{2}{>{\centering\arraybackslash}p{3.01cm}}{$\{0,10\}$} & \multicolumn{2}{>{\centering\arraybackslash}p{3.01cm}}{$10\times$ matching score, in $[0,10]$} & \multicolumn{2}{>{\centering\arraybackslash}p{3.01cm}@{}}{$\{0,1\}$}\\
Discount factor & \multicolumn{6}{>{\raggedright\arraybackslash}p{9.45cm}@{}}{$\gamma=0.95$}\\
Step-channel weight $\omega$ & \multicolumn{2}{>{\centering\arraybackslash}p{3.01cm}}{$0.8$} & \multicolumn{2}{>{\centering\arraybackslash}p{3.01cm}}{$0.8$} & \multicolumn{2}{>{\centering\arraybackslash}p{3.01cm}@{}}{$0.9$}\\
Advantage normalization & \multicolumn{2}{>{\centering\arraybackslash}p{3.01cm}}{Mean centring} & \multicolumn{2}{>{\centering\arraybackslash}p{3.01cm}}{Mean centring} & \multicolumn{2}{>{\centering\arraybackslash}p{3.01cm}@{}}{Mean centring and standard-deviation scaling}\\
Anchor similarity threshold & \multicolumn{2}{>{\centering\arraybackslash}p{3.01cm}}{$0.95$} & \multicolumn{2}{>{\centering\arraybackslash}p{3.01cm}}{$0.95$} & \multicolumn{2}{>{\centering\arraybackslash}p{3.01cm}@{}}{$0.9$}\\
\midrule
\multicolumn{7}{@{}l}{\textit{Self-diagnosis and penalty allocation}}\\*
\midrule
Diagnosis sampling & \multicolumn{6}{>{\raggedright\arraybackslash}p{9.45cm}@{}}{$M=3$ per trajectory; temperature $0.7$, top-$p=1$}\\
Maximum errors per diagnosis & \multicolumn{6}{>{\raggedright\arraybackslash}p{9.45cm}@{}}{$5$}\\
Diagnosis agreement threshold & \multicolumn{6}{>{\raggedright\arraybackslash}p{9.45cm}@{}}{$\ge2$ parsable diagnoses; $\ge2$ votes per L1 hit}\\
Score features & \multicolumn{2}{>{\centering\arraybackslash}p{3.01cm}}{L2 error rates} & \multicolumn{2}{>{\centering\arraybackslash}p{3.01cm}}{L2 error rates} & \multicolumn{2}{>{\centering\arraybackslash}p{3.01cm}@{}}{L1 hit indicators}\\
First-error boost in score & \multicolumn{2}{>{\centering\arraybackslash}p{3.01cm}}{$b=1.5$} & \multicolumn{2}{>{\centering\arraybackslash}p{3.01cm}}{$b=1.5$} & \multicolumn{2}{>{\centering\arraybackslash}p{3.01cm}@{}}{Not used}\\
First-error boost in allocation & \multicolumn{6}{>{\raggedright\arraybackslash}p{9.45cm}@{}}{$b_{\mathrm{loc}}=1.5$ for all six runs}\\
Rule-detector allocation weights & \multicolumn{6}{>{\raggedright\arraybackslash}p{9.45cm}@{}}{Repeated failed action: $1$; invalid action: $0.5$; action cycle: $1$}\\
Warm-up strength $\lambda_{\mathrm{warm}}$ & $28.156$ & $26.5$ & $8.141$ & $12.37$ & $0.2$ & $0.2$\\
Peak strength $\lambda_{\mathrm{peak}}$ & $84.468$ & $79.5$ & $8.141$ & $37.11$ & $0.6$ & $0.6$\\
Penalty schedule & \multicolumn{6}{>{\raggedright\arraybackslash}p{9.45cm}@{}}{$U_{\mathrm w}=10$ warm-up updates; cosine schedule with final strength $0$ at update $160$}\\
Budget cap $\eta\dR$ ($\eta=0.9$) & \multicolumn{2}{>{\centering\arraybackslash}p{3.01cm}}{$9$} & \multicolumn{2}{>{\centering\arraybackslash}p{3.01cm}}{$9$} & \multicolumn{2}{>{\centering\arraybackslash}p{3.01cm}@{}}{$0.9$}\\
\midrule
\multicolumn{7}{@{}l}{\textit{Initial error-price estimation}}\\*
\midrule
Pricing classes $K$ & $24$ & $20$ & $32$ & $18$ & $10$ & $8$\\
Estimator & \multicolumn{2}{>{\centering\arraybackslash}p{3.01cm}}{Conditional logit with Firth correction} & \multicolumn{2}{>{\centering\arraybackslash}p{3.01cm}}{Within-group ridge} & \multicolumn{2}{>{\centering\arraybackslash}p{3.01cm}@{}}{Conditional logit with Firth correction}\\
Cold-start shrinkage $\rho$ & \multicolumn{2}{>{\centering\arraybackslash}p{3.01cm}}{$0.5$} & \multicolumn{2}{>{\centering\arraybackslash}p{3.01cm}}{$0.003$} & \multicolumn{2}{>{\centering\arraybackslash}p{3.01cm}@{}}{$0.5$}\\
Length covariate in initial fit & \multicolumn{2}{>{\centering\arraybackslash}p{3.01cm}}{$T_{q,i}/T_{\max}$} & \multicolumn{2}{>{\centering\arraybackslash}p{3.01cm}}{Included} & \multicolumn{2}{>{\centering\arraybackslash}p{3.01cm}@{}}{None}\\
Initial price scaling & \multicolumn{2}{>{\centering\arraybackslash}p{3.01cm}}{Feature standard deviation $\sigma_k$} & \multicolumn{2}{>{\centering\arraybackslash}p{3.01cm}}{Feature standard deviation $\sigma_k$} & \multicolumn{2}{>{\centering\arraybackslash}p{3.01cm}@{}}{No feature scaling}\\
Coefficient bounds & \multicolumn{2}{>{\centering\arraybackslash}p{3.01cm}}{$[-8,8]$} & \multicolumn{2}{>{\centering\arraybackslash}p{3.01cm}}{Unbounded} & \multicolumn{2}{>{\centering\arraybackslash}p{3.01cm}@{}}{$[-8,8]$}\\
\midrule
\multicolumn{7}{@{}l}{\textit{Online error-price estimation}}\\*
\midrule
Refit window / schedule & \multicolumn{6}{>{\raggedright\arraybackslash}p{9.45cm}@{}}{$W=50$ updates; first refit at $u_0=30$, then every $\Delta u=10$ updates}\\
Refit support & \multicolumn{6}{>{\raggedright\arraybackslash}p{9.45cm}@{}}{At least $30$ nonzero records per fitted class and $20$ informative groups; binary fits also require total successes in these groups $\ge10\times$ the number of fitted classes}\\
Online shrinkage & \multicolumn{6}{>{\raggedright\arraybackslash}p{9.45cm}@{}}{$\rho_u=1$ through update $20$; linear decay to $0.1$ at update $100$}\\
Coefficient EMA & \multicolumn{6}{>{\raggedright\arraybackslash}p{9.45cm}@{}}{$\nu=0.7$ (old); $1-\nu=0.3$ (new)}\\
Online price scaling & \multicolumn{6}{>{\raggedright\arraybackslash}p{9.45cm}@{}}{$\sigma_k=1$}\\

\end{longtable}
\endgroup

\subsection{Diagnosis Prompt}
\label{app:prompts}

The prompt sequence follows Appendix~\ref{app:coldstart}: the teacher first names errors without a taxonomy, an induction prompt builds and freezes the two-level table, and the diagnosis prompt then uses that table to label trajectories.
Only the last prompt runs during RL (Figure~\ref{fig:prompt-diag}).
A single rendering function supplies the same prompt format for diagnosis SFT, online self-diagnosis, acceptance tests and frozen-checkpoint evaluations.
This prompt serves the role of the \textsc{SEED} analyzer (their Figure~10), using a learned, verifiable error table in place of free-text analysis.
The teacher and induction prompts, one-line environment descriptions and actor prompts are released with our code; the ALFWorld actor prompt is taken unchanged from \textsc{SEED}'s Figure~11.

\FloatBarrier
\section{Additional Experimental Results}
\label{app:exp-results}
We examine signal coverage and strength, step-level credit allocation and cold-start data scale.
Unless stated otherwise, results use ALFWorld training rollouts from one run per method; GRPO (diag-SFT) and GiGPO (diag-SFT) share \method{}'s diagnosis-SFT initialization.
Apart from the cold-start scale study, these analyses describe diagnosis and training behavior; task performance is reported in Tables~\ref{tab:main} and~\ref{tab:ablation}.

\subsection{Learning Signal in Same-Outcome Groups}
\label{app:signal-figs}
We first define the measurements used in \S\ref{sec:analysis-signal}, then separate the algebraic effect of equal outcomes from the observed coverage.

\subsubsection{Coverage and Relative Signal Strength}
\paragraphhd{Groups and measured quantities.}
A group $g$ contains $G=8$ rollouts of one task at one update.
Each run has $16$ groups per update over $160$ updates.
Write $c(g)\in\{\mathrm{mix},\mathrm{succ},\mathrm{fail}\}$ for mixed, all-success and all-fail groups, respectively, and let $N_c$ be the number of groups of type $c$, with $N=\sum_cN_c$.
For coverage, the measured credit is
\begin{equation}
\kappa_s=
\begin{cases}
A_{q,i,t}, & \text{GRPO/GiGPO: step sample }s=(i,t),\\[2pt]
S_{q,i}, & \text{\method{}: trajectory sample }s=i.
\end{cases}
\label{eq:ruler-credit}
\end{equation}
Here $A_{q,i,t}$ is the optimizer advantage, whereas $S_{q,i}$ is the diagnostic score of Appendix~\ref{app:budget}, before multiplication by $\lambda_u$.
Thus \method{}'s coverage measures available diagnostic contrast, not the magnitude of its final policy advantage.

\paragraphhd{Coverage.}
Define the within-group range and the usable-signal indicator by
\begin{equation}
r_g=\max_{s\in g}\kappa_s-\min_{s\in g}\kappa_s,
\qquad
\mathrm{usable}(g)=\ind{r_g>\tau\bar r_{\mathrm{mix}}},
\label{eq:usable}
\end{equation}
where each method uses its own mixed-group mean:
\begin{align}
\bar r_{\mathrm{mix}}
&=\frac{1}{N_{\mathrm{mix}}}\sum_{g:\,c(g)=\mathrm{mix}}r_g,
\label{eq:ruler-norm}\\
\mathrm{Coverage}(\tau)
&=\frac{1}{N}\sum_g\ind{r_g>\tau\bar r_{\mathrm{mix}}}.
\label{eq:coverage}
\end{align}
The measured runs have $N_{\mathrm{mix}}>0$ and $\bar r_{\mathrm{mix}}>0$.
Mixed groups supply outcome contrast and set a within-method reference scale.
We use $\tau=0.05$; ``usable'' denotes a contrast above this threshold, not independently verified correctness or a guaranteed policy improvement.

\paragraphhd{Signal-retention Index.}
For signal strength, let $r_g^{\mathrm I}$ use the same range definition, but replace $S_{q,i}$ by the injected budget $P_{q,i}$ for \method{}; the baselines still use $A_{q,i,t}$.
In a $20$-update window $w$, let $N_{w,c}$ count groups of type $c$ and let $\bar r^{\mathrm I}_{w,c}$ be their mean range.
For a positive mixed-group mean, define
\begin{equation}
I_w=\sum_{c\in\{\mathrm{mix},\mathrm{succ},\mathrm{fail}\}}
\frac{N_{w,c}}{N_w}\,
\frac{\bar r^{\mathrm I}_{w,c}}{\bar r^{\mathrm I}_{w,\mathrm{mix}}},
\qquad N_w=\sum_cN_{w,c}.
\label{eq:signal-index}
\end{equation}
An absent group type contributes zero.
Each term combines a type's frequency with its contrast relative to mixed groups; the mixed contribution is exactly its frequency.
Figure~\ref{fig:signal_index_curve} normalizes each window separately.
The run-pooled Index is recomputed from all groups, rather than averaged across windows.
Because \method{} uses trajectory budgets and the baselines use step advantages, this is a comparison of contrast proxies, not a common-scale measure of optimizer credit.
Redistribution and centring can change how budget differences enter the update, so the proxy alone does not provide an upper bound on advantage contrast.

\subsubsection{Credit under Equal Terminal Rewards}
ALFWorld gives terminal reward $\Rmax=10$ for success and $\Rmin=0$ for failure.
The following result isolates terminal rewards and diagnostic penalties; it omits the shared invalid-action penalty and other reward terms (Appendices~\ref{app:adv} and~\ref{app:case}).
For a fixed task $q$, use the step set $\Omega_q$ and anchor groups of Appendix~\ref{app:adv}.
For any non-empty comparison set $\mathcal G\subseteq\Omega_q$, write
$\bar\delta_{\mathcal G}=|\mathcal G|^{-1}\sum_{(j,r)\in\mathcal G}\delta_{q,j,r}$.
For $\calH=\calH(o_{q,i,t})$, define the centred discount factor
\begin{equation}
\psi_{q,i,t}=\gamma^{T_{q,i}-t}
-\frac{1}{|\calH|}\sum_{(j,r)\in\calH}\gamma^{T_{q,j}-r}.
\label{eq:psi}
\end{equation}

\begin{proposition}[Credit in same-outcome groups]
\label{prop:same-outcome}
Fix a rollout group with purely terminal rewards $\Rout_{q,i}=\bar R$ for all $i$, discount $0<\gamma<1$, and step-channel weight $\omega\ge0$.
Assume that the non-empty anchor groups form a fixed partition of $\Omega_q$, and that normalization scales are finite and strictly positive.
Omitting additional reward terms, the following hold.
\begin{enumerate}[label=(\roman*),leftmargin=*]
\item GRPO and the GiGPO episode channel have zero outcome advantage:
\begin{equation}
A^{\mathrm{GRPO}}_{q,i}=A^{\mathrm{GiGPO},\mathrm{ep}}_{q,i}=0.
\label{eq:prop2-grpo}
\end{equation}
\item Within an anchor group $\calH$, GiGPO's step advantage is
\begin{equation}
A^{\mathrm{GiGPO},\mathrm{step}}_{q,i,t}
=\frac{\bar R}{\xi_{\calH}}\,\psi_{q,i,t},
\qquad \xi_{\calH}>0,
\label{eq:prop2-gigpo}
\end{equation}
where $\xi_{\calH}$ is the group's common normalization scale.
It vanishes everywhere when $\bar R=0$.
For $\bar R\ne0$, it vanishes throughout $\calH$ if and only if every member has the same remaining length $T_{q,j}-r$.
\item \method{} has the channel and joint advantages
\begin{equation}
\begin{aligned}
\Aep_{q,i,t}&=-\big(\delta_{q,i,t}-\bar\delta_{\Omega_q}\big),\\
\Astep_{q,i,t}&=\bar R\,\psi_{q,i,t}-\big(\delta_{q,i,t}-\bar\delta_{\calH}\big),\\
A_{q,i,t}&=\omega\bar R\,\psi_{q,i,t}
-\big(\delta_{q,i,t}-\bar\delta_{\Omega_q}\big)
-\omega\big(\delta_{q,i,t}-\bar\delta_{\calH}\big).
\end{aligned}
\label{eq:prop2-dtcr}
\end{equation}
When $\bar R=0$, the joint advantage is zero at every sample if and only if $\delta_{q,i,t}$ is constant over $\Omega_q$.
\end{enumerate}
\end{proposition}

\begin{proof}
For (i), let $\mu_q=G^{-1}\sum_i\Rout_{q,i}$ and
$\sigma_q=\sd(\Rout_{q,1},\ldots,\Rout_{q,G})$.
Equal rewards give $\mu_q=\bar R$ and $\sigma_q=0$, hence
\[
A^{\mathrm{GRPO}}_{q,i}
=\frac{\Rout_{q,i}-\mu_q}{\sigma_q+\zeta}
=\frac{\bar R-\bar R}{\zeta}=0,
\qquad \zeta>0.
\]
The GiGPO episode channel uses the same centred outcome.
For (ii), purely terminal rewards give
\[
Z_{q,i,t}=\sum_{r=t}^{T_{q,i}}\gamma^{r-t}R^{\mathrm{env}}_{q,i,r}
=\bar R\gamma^{T_{q,i}-t},
\qquad
\frac{Z_{q,i,t}-\bar Z_{\calH}}{\xi_{\calH}}
=\frac{\bar R}{\xi_{\calH}}\psi_{q,i,t},
\]
where $\bar Z_{\calH}$ is the anchor-group mean.
For $\bar R\ne0$, all centred returns vanish exactly when all discount factors agree.
Since $d\mapsto\gamma^d$ is strictly decreasing for $0<\gamma<1$, this is equivalent to equal remaining lengths.

For (iii), substituting the common reward into Eqs.~\eqref{eq:aep}--\eqref{eq:astep} gives
$E_{q,i,t}=\bar R-\delta_{q,i,t}$ and
$\Zt_{q,i,t}=\bar R\gamma^{T_{q,i}-t}-\delta_{q,i,t}$.
Centring these scores over $\Omega_q$ and $\calH$, respectively, and adding the channels gives Eq.~\eqref{eq:prop2-dtcr}.
When $\bar R=0$, write $m=\bar\delta_{\Omega_q}$ and $m_{\calH}=\bar\delta_{\calH}$, so
\[
A_{q,i,t}=-(m_{\calH}-m)-(1+\omega)(\delta_{q,i,t}-m_{\calH}).
\]
If all advantages are zero, averaging within each anchor group gives $m_{\calH}=m$; since $1+\omega>0$, every share then equals $m$.
The converse follows by substitution.
\end{proof}

The result concerns advantage contrast, not semantic correctness or a guaranteed gradient improvement.
A singleton anchor group has zero step-channel advantage for both GiGPO and \method{}, although \method{}'s episode channel can remain non-zero.
Equal outcomes do not force diagnostic contrast to vanish, but uniform shares do; different error labels need not yield different shares.
In all-success groups, discount and diagnostic terms can also cancel.

\subsubsection{Observed Coverage and Strength}
\paragraphhd{Coverage and its source.}
At $\tau=0.05$, every mixed group is usable for all three methods, so coverage differences come from same-outcome groups.
GRPO's residual format contribution never exceeds the threshold in these groups.
GiGPO can distinguish all-success samples at different remaining distances from a shared anchor context to success, favoring shorter paths.
Its overall coverage falls below GRPO's near $\tau=0.25$ and approaches its mixed-group share at larger thresholds.
Its all-fail groups have no usable contrast.
Under the diagnostic-score measurement, \method{} covers every all-fail group and $90\%$ of all-success groups.
Of the $125$ all-success groups below threshold, $53$ have zero scores in all eight rollouts.
These are empirical coverage results, not guarantees implied by different diagnoses.
Figure~\ref{fig:signal_over_training} expands Figure~\ref{fig:teaser}(a) over training windows; Figure~\ref{fig:signal_tau_index}(a) varies the threshold.

\paragraphhd{Available versus applied contrast.}
Replacing $S_{q,i}$ by $P_{q,i}$ in the coverage calculation reduces \method{}'s run-wide coverage from $95.1\%$ to $84.6\%$.
The two measurements differ by at most two percentage points through the first $120$ updates.
In the last $20$ updates, as $\lambda_u$ approaches zero, only $20\%$ of groups retain a budget range above threshold.
Thus diagnosis can still distinguish trajectories while the applied penalty becomes small.

\paragraphhd{Interpreting the Index.}
Figure~\ref{fig:signal_tau_index}(b,c) gives the pooled Index and its same-outcome contributions; Figure~\ref{fig:signal_index_curve} gives the windowed values.
\method{}'s pooled $0.821$ is a budget-based proxy, subject to the measurement distinction above.
For sensitivity only, hold group frequencies and same-outcome mean ranges fixed while multiplying the mixed-group mean range by a hypothetical factor $k\in[2,3]$.
The Index then becomes approximately $0.372+0.449/k\in[0.52,0.60]$, still above the baselines' approximately $0.42$.
This is neither a measured advantage Index nor a proven bound.

\begin{figure}[t]
\centering
\includegraphics[width=\textwidth]{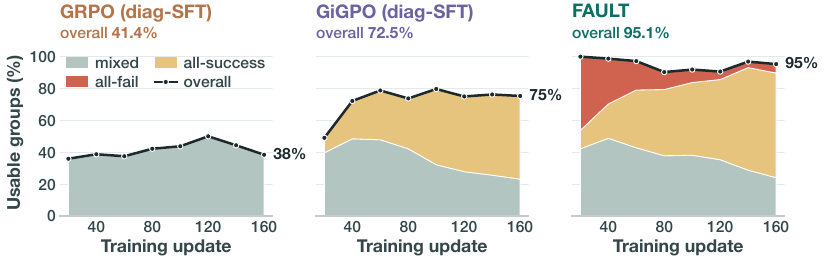}
\caption{\textbf{Usable signal over training.} Stacked shares of all groups with a usable range at $\tau=0.05$, measured in $20$-update windows at their final update. The dark line is total coverage; titles report run-wide values. Same-outcome groups (red and gold) dominate and grow, with \method{} retaining broad coverage of them. \method{} uses scores $S$ and the baselines use optimizer advantages. Single-seed runs; both baselines start from diagnosis SFT.}
\label{fig:signal_over_training}
\end{figure}

\begin{figure}[t]
\centering
\includegraphics[width=\textwidth]{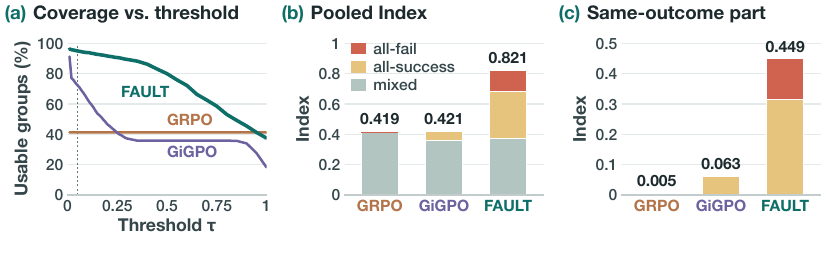}
\caption{\textbf{Threshold sensitivity and pooled signal strength.} (a) Coverage as $\tau$ increases, using each method's mixed-group mean; the dotted line marks $\tau=0.05$. (b) Pooled Index split by group type; the mixed contribution equals its group share. (c) Same-outcome contributions on a finer scale. Coverage uses $S$ for \method{}, whereas its Index uses $P$; baselines use advantages for both. These measure contrast, not independently verified correctness.}
\label{fig:signal_tau_index}
\end{figure}

\subsection{Penalties on Successful Trajectories}
\label{app:audit}
\paragraphhd{Measurement.}
Success certifies the terminal outcome but does not exempt a trajectory from diagnostic penalties.
Figure~\ref{fig:success_charged}(a) separates three rates: verified-error reports and L1-majority reports are fractions of valid diagnoses, whereas the charge rate is the fraction of successful trajectories with $P>0.001$ reward units.
Verification checks the cited quote and step; it is not an independent judgment of semantic correctness.

\paragraphhd{Results and the budget cap.}
Across training, $61.85\%$ of successes ($8{,}719/14{,}097$) receive a charge above this threshold, decreasing from $80.5\%$ in the first window to $52.2\%$ in the last.
Failures have a higher median budget than successes in every window (Figure~\ref{fig:success_charged}(b)).
The cap of Eq.~\eqref{eq:cap} preserves success--failure ordering in the diagnostic episode channel by Proposition~\ref{prop:order}.
It binds on $56/20{,}480$ trajectories ($0.27\%$, all successes) and never removes the entire diagnostic-score contrast of a complete group (Figure~\ref{fig:success_charged}(c)).

\begin{figure}[t]
\centering
\includegraphics[width=\textwidth]{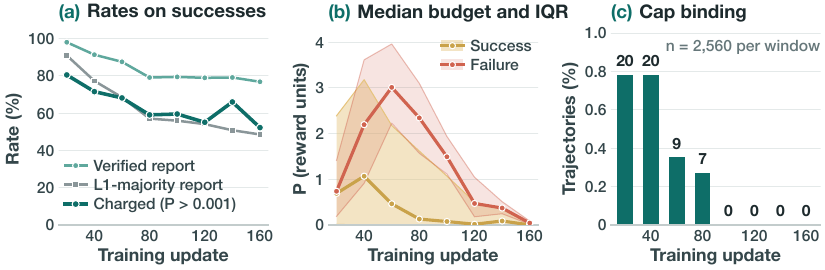}
\caption{\textbf{Diagnostic penalties on successful trajectories.} Each point summarizes $20$ updates. (a) Verified and L1-majority report rates among valid diagnoses, and charge rate ($P>0.001$) among all successes. (b) Median budget $P$ and interquartile range for successes and failures. (c) Fraction of the $2{,}560$ trajectories per window reaching the cap; labels give counts. Windows are placed at their final update. Diagnoser and prices evolve during training; single seed.}
\label{fig:success_charged}
\end{figure}

\subsection{Step-Level Localization: Measurement and Additional Views}
\label{app:localization}
This subsection defines the localization and concentration measures used in \S\ref{sec:analysis-local}, Figures~\ref{fig:teaser}(b) and~\ref{fig:localization}, and the additional views in Figure~\ref{fig:localization-extra}.

\subsubsection{Evaluation Protocol}
\paragraphhd{Step penalties.}
For a trajectory of $T=T_{q,i}$ steps, define the non-negative credit assigned against step $t$ as
\begin{equation}
b_t=
\begin{cases}
\delta_{q,i,t}, & \text{\method{}},\\[2pt]
\max(-A_{q,i,t},0), & \text{GRPO and GiGPO},
\end{cases}
\label{eq:blame}
\end{equation}
where $A_{q,i,t}$ is the optimizer advantage.
GRPO broadcasts its outcome advantage across steps; the additional invalid-action penalty makes $b_t$ constant or two-valued.

\paragraphhd{Blind reference labels.}
We use $300$ failed training trajectories of \method{}, $300$ from a same-configuration GiGPO rerun, and $48$ from GRPO, sampled evenly across eight $20$-update windows and from both all-fail and mixed groups.
Every failure reaches the $T=30$ step limit.
An independent LLM, Qwen3.8-Max~\citep{qwen2026qwen38max}, sees only the action and feedback text, without method names or penalties, and selects the earliest decisive step $j$ whose correction it judges most likely to turn failure into success.
These are model judgments, not human or counterfactually verified labels.
A blind relabel of $48$ trajectories agrees within one step in $75\%$ of cases.

\subsubsection{Localization Metrics and Results}
Rank steps by decreasing $b_t$, breaking ties uniformly at random.
Allowing one step of error around the judge's label, define
\begin{equation}
\mathcal J=\{t\in\{1,\ldots,T\}:|t-j|\le1\},
\qquad
\mathrm{rank}^{\star}=\min_{t\in\mathcal J}\mathrm{rank}(t).
\label{eq:judge-rank}
\end{equation}
The ranking and signed-offset metrics are
\begin{equation}
\mathrm{MRR}=\E\!\left[\frac{1}{\mathrm{rank}^{\star}}\right],
\qquad
\mathrm{hit@}k=\Pr(\mathrm{rank}^{\star}\le k),
\qquad
\mathrm{off}=\min\!\left(\argmax_t b_t\right)-j.
\label{eq:rank-scores}
\end{equation}
MRR and hit@$k$ average over trajectories and $20{,}000$ random tie-breaks per trajectory; the offset instead uses the first maximizer.
Positive offsets place the maximum penalty after the judged step.
Constant $b_t$ gives the uniform random-ranking reference.

\paragraphhd{Localization results.}
\method{} achieves MRR $0.494$, hit@$1$ $34.1\%$, and $82$ exact step matches; GiGPO achieves $0.303$, $11.9\%$, and $21$ exact matches.
On its $48$ trajectories, GRPO achieves MRR $0.306$ and hit@$1$ $15.3\%$, compared with $0.266$ and $9.9\%$ for random ranking.
Figure~\ref{fig:teaser}(b) places GRPO at the random reference for its uniform outcome credit; the measured value includes the invalid-action penalty, which sometimes marks the decisive step.
In $73\%$ of GRPO trajectories, the maximum penalty is tied, so its offset distribution is omitted from Figure~\ref{fig:localization}(a).
GiGPO's median offset is $+5$ steps.
Figure~\ref{fig:localization-extra}(a) extends hit@$k$ to $k=1,\ldots,12$: \method{} leads most at small $k$, while the curves approach each other as $k$ grows.
Even random nomination of $10$ out of $30$ steps covers the tolerance window about $72\%$ of the time.

\subsubsection{Concentration of Step Credit}
For \method{}, set $c_t=\delta_{q,i,t}$; for the baselines, set $c_t=A_{q,i,t}$.
For $T\ge2$ and $\sum_t|c_t|>0$, define
\begin{equation}
w_t=\frac{|c_t|}{\sum_{r=1}^{T}|c_r|},
\qquad
H^*=-\frac{1}{\ln T}\sum_{t:w_t>0}w_t\ln w_t,
\qquad
\mathrm{ENS}=\left(\sum_{t=1}^{T}w_t^2\right)^{-1}.
\label{eq:concentration}
\end{equation}
The normalized entropy $H^*\in[0,1]$ is zero for concentration on one step and one for a uniform allocation.
The effective number of steps, ENS, is the inverse Herfindahl index: it equals $1$ and $T$ in these two cases.
For \method{} these statistics describe the diagnostic allocation, not its full advantage, which also depends on returns and comparison-group baselines; the baseline statistics use absolute optimizer advantages.

\paragraphhd{Excluded trajectories and training trends.}
Zero-total-credit trajectories have no normalized distribution and are excluded.
They account for $31.4\%$ of \method{}'s same-outcome trajectories, mainly because no verified error with a non-zero price was found, versus $7.9\%$ for GRPO and $0.1\%$ for GiGPO.
In same-outcome groups, GRPO has at most two step-advantage values separated by $0.1$, reflecting the invalid-action penalty.
Figure~\ref{fig:localization-extra}(b) tracks mean ENS in all-fail groups, where $T=30$ and a uniform allocation has ENS $30$.
Over training, ENS falls from $26.5$ to $15.6$ for GRPO, $24.0$ to $11.3$ for GiGPO, and $11.3$ to $3.3$ for \method{}.
The decline for \method{} partly reflects fewer diagnosed errors as penalized behaviors recede (Appendix~\ref{app:dose-response}); concentration alone does not establish more accurate localization.

\begin{figure}[t]
\centering
\includegraphics[width=\textwidth]{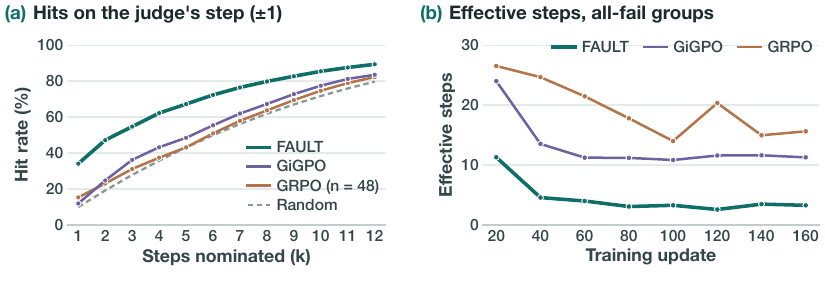}
\caption{\textbf{Additional localization and concentration results.} (a) Hit@$k$ within one step of the blind judge's label, using the trajectories of Figure~\ref{fig:localization}(a) and $48$ GRPO trajectories; the dashed curve is random ranking. (b) Mean ENS on all-fail groups in $20$-update windows, placed at their final update. Both baselines start from diagnosis SFT.}
\label{fig:localization-extra}
\end{figure}

\subsection{Penalties on Wasted Steps}
\label{app:mech-probe}
\paragraphhd{Step classes and measurement.}
We use two mechanically identified classes: invalid or missing actions, and format-valid no-ops whose next feedback is ``Nothing happens.''
Invalid actions take precedence; the last step cannot be labeled a no-op without a subsequent observation.
All failed trajectories reach the $30$-step limit, so these steps consume the interaction budget without recorded progress.
The labels identify wasted interactions, not whether correcting one step would have changed the outcome.
For the steps $\mathcal B$ of one class in a trajectory, its targeting ratio is
\[
\frac{\sum_{t\in\mathcal B}b_t/\sum_{t=1}^{T}b_t}{|\mathcal B|/T},
\]
using Eq.~\eqref{eq:blame}.
Ratios and top-1 rates use failed trajectories with positive total negative credit and class counts strictly between $0$ and $T$; eligibility therefore depends on the method and class.
We average the ratio equally over these trajectories, with one denoting uniform allocation.
Top-1 hits average uniformly over tied maximizers of $b_t$.

\paragraphhd{Invalid actions.}
Among eligible trajectories, both baselines achieve $100\%$ top-1 targeting of invalid actions in all-fail groups; over the pooled all-fail and mixed-group sample, GRPO reaches $100\%$ compared with $48.6\%$ for \method{}.
The shared invalid-action penalty directly marks these steps, so this result alone does not demonstrate diagnostic understanding.
In the recorded recipe, an invalid action receives $-0.1$ and a valid action receives zero before advantage construction (Appendix~\ref{app:case}).

\paragraphhd{Format-valid no-ops.}
\method{} assigns $3.67\times$ the uniform share of negative credit to no-ops, compared with $0.48\times$ for GRPO and $0.57\times$ for GiGPO (Figure~\ref{fig:mech_gold}).
The baselines' no-op top-1 hit rates are $0.1\%$ and $4.3\%$, below their uniform references of $7.2\%$ and $7.7\%$.
In the analysed all-fail groups, their no-op negative credit is zero, although a valid step can still have a non-negative advantage after centring.
Thus \method{} targets a class of wasted steps not directly flagged by the invalid-action reward; the contrast goes beyond locating format violations.

\begin{figure}[t]
\centering
\includegraphics[width=\textwidth]{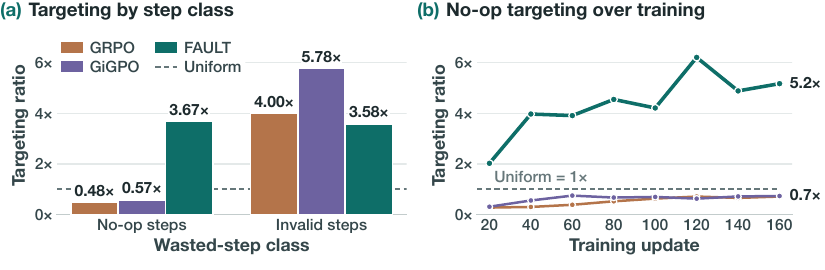}
\caption{\textbf{Negative-credit allocation to wasted steps.} (a) Mean targeting ratio for invalid actions and format-valid no-ops under shared mechanical labels; $1\times$ is uniform allocation. (b) No-op targeting over training. Credit is $\delta_t$ for \method{} and $\max(-A_t,0)$ for the baselines. All analysed failures reach the step limit. Baselines: GRPO (diag-SFT) and GiGPO (diag-SFT); offline analysis of single-seed runs.}
\label{fig:mech_gold}
\end{figure}

\subsection{Cold-Start Data Scale}
\label{app:scale}
\paragraphhd{Cold-start scale protocol.}
Table~\ref{tab:scale} compares three cold-start pool sizes with Qwen3-4B-Instruct-2507.
The ALFWorld pools are nested, adding $1{,}440$ trajectories each from the base model, a checkpoint after $20$ GRPO updates and a checkpoint after $40$ updates; pool size and policy composition therefore vary together.
Diagnosis F1 is evaluated on $600$ trajectories from $25$ unseen ALFWorld tasks using the same $17$-class reference set.
The number of non-zero prices counts pricing classes with a negative cold-start coefficient.
Pre-RL success is measured with sampled decoding at temperature $0.4$; the raw backbone scores $17.9$.
The ALFWorld RL runs are evaluated after $160$ updates using an earlier protocol: $134$ draws sampled with replacement from the $134$ unseen tasks, with greedy decoding for each draw.
Their success rates therefore count successful draws, rather than distinct tasks.
The three WebShop runs use the earlier binary-reward recipe and report post-RL task score and strict success rate.

\begin{table}[!htbp]
\centering
\footnotesize
\caption{Cold-start data scale with Qwen3-4B-Instruct-2507. ALFWorld reports named-class diagnosis F1 of the SFT model, non-zero initial prices, pre-RL SFT success rate and success rate after $160$ RL updates. WebShop reports post-RL task score and strict success rate. Evaluation protocols are specified above.}
\label{tab:scale}
\begin{tabular}{@{}l cccc cc@{}}
\toprule
& \multicolumn{4}{c}{ALFWorld} & \multicolumn{2}{c}{WebShop} \\
\cmidrule(lr){2-5}\cmidrule(lr){6-7}
Cold-start pool & Diag.\ F1 & Non-zero prices & SFT SR & RL SR (draws) & Score & Succ. \\
\midrule
$1{,}440$ & $0.617$ & $6/20$ & $24.6$ & $74.6$ & $73.8$ & $60.2$ \\
$2{,}880$ & $0.648$ & $8/20$ & $31.3$ & $82.1$ & $85.6$ & $67.6$ \\
$4{,}320$ & $0.668$ & $3/20$ & $30.6$ & $83.6$ & $82.3$ & $70.3$ \\
\bottomrule
\end{tabular}
\end{table}

\section{Training Dynamics}
\label{app:dynamics}
We examine training trajectory length, the frequency of penalized behaviors, and the evolution of error prices on ALFWorld.
These analyses use training rollouts and fitted coefficients, rather than additional test evaluations.
The frozen-price variant is the w/o online pricing variant in Table~\ref{tab:ablation}.

\subsection{Training Trajectory Length}
\label{app:length-dynamics}

\paragraphhd{Measurement.}
Trajectory length $T_{q,i}$ counts active environment interactions, including invalid actions but excluding padding; it measures neither token count nor elapsed time.
Figure~\ref{fig:traj_length} reports the mean over successful and failed trajectories at each update; Table~\ref{tab:length-windows} averages these update means over the final $5$, $10$, $20$ or $40$ updates.
The four runs are raw-initialized GRPO, SEED, the frozen-price variant and the full method, all observed through update $160$ with a $30$-step horizon.
Equal update counts do not imply equal numbers of interactions or matched training configurations.

\begin{table}[!htbp]
\centering
\footnotesize
\setlength{\tabcolsep}{4pt}
\caption{Mean training trajectory length over four windows ending at update $160$, the last update shared by all runs. Each value equally averages the per-update mean number of active interactions. All runs use a $30$-step horizon; these are single-run training summaries, not test-time estimates or confidence intervals.}
\label{tab:length-windows}
\begin{tabular}{@{}lrrrr@{}}
\toprule
Updates & \makecell{GRPO\\raw-init} & SEED & \method{} & \makecell{w/o online\\pricing} \\
\midrule
156--160 & 21.89 & 14.54 & 11.43 & 12.03 \\
151--160 & 21.69 & 13.77 & 12.03 & 11.38 \\
141--160 & 22.09 & 14.34 & 12.56 & 11.70 \\
121--160 & 22.21 & 14.38 & 12.64 & 12.00 \\
\bottomrule
\end{tabular}
\end{table}

\paragraphhd{Results across windows.}
The full method produces shorter trajectories than GRPO and SEED in all four windows.
The frozen-price variant, which also removes the cap, is shorter than the full method by about $0.6$--$0.9$ steps in three windows; the ordering reverses by about $0.6$ steps in the final five updates.
It has the shortest mean over updates $151$--$160$, but it solves fewer tasks in the separate greedy evaluation ($120$ versus $122$ of $134$).
Thus, shorter training trajectories need not correspond to higher task success.
These single-run comparisons show sensitivity to the summary window, not statistical significance or the effect of online pricing alone.

\paragraphhd{Interpreting mean length.}
A lower mean can reflect shorter successful paths or more successes replacing horizon-truncated failures.
These aggregate summaries do not separate lengths by outcome, so they cannot distinguish the two effects.
The curves therefore describe training behavior without establishing that diagnostic redistribution alone improves path efficiency.

\subsection{Penalized Behaviors over Training}
\label{app:dose-response}

\paragraphhd{Measurement.}
We track invalid or missing actions and format-valid no-ops using the mechanical labels of Appendix~\ref{app:mech-probe}.
In each $20$-update window, Figure~\ref{fig:dose_response}(a) aggregates \method{}'s diagnostic penalty mass by class, while panels~(b,c) divide each class's step count by the total recorded step count.
The comparisons include \method{}, GRPO (diag-SFT) with the same initialization, and SEED as a separate reference; GiGPO lacks the rollout text needed for this analysis.

\begin{figure}[!htbp]
\centering
\includegraphics[width=\textwidth]{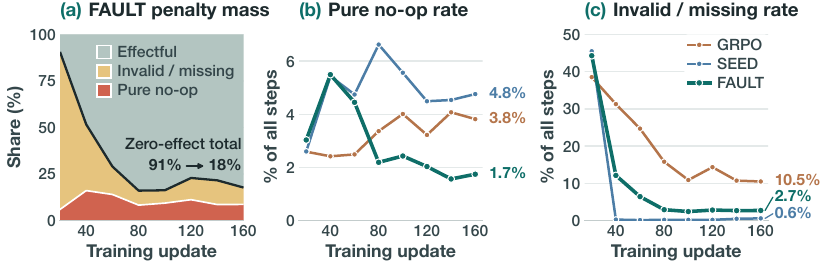}
\caption{\textbf{Penalty allocation and behavior frequencies during training.} Points summarize $20$-update windows and are placed at the final update. (a) Share of \method{}'s diagnostic penalty mass assigned to each mechanical step class. (b) Format-valid no-op rate. (c) Invalid or missing action rate. Rates use all recorded steps as the denominator; the remaining-step category in~(a), labeled effectful, includes steps without subsequent feedback. GRPO starts from diagnosis SFT; SEED is a reference with a different training mechanism. GiGPO is omitted because its original logs lack rollout text. Single-seed observational comparisons.}
\label{fig:dose_response}
\end{figure}

\paragraphhd{Observed changes.}
The combined share of \method{}'s penalty mass on invalid actions and no-ops falls from $91\%$ to $18\%$ as these behaviors become less frequent.
Invalid actions, which receive a shared format penalty, decline under all three methods.
Format-valid no-ops are not directly identified by this penalty: their rate falls from $3.0\%$ to $1.7\%$ under \method{}, but rises from $2.6\%$ to $3.8\%$ under GRPO and from $2.6\%$ to $4.8\%$ under SEED.
Together with the allocation results in Appendix~\ref{app:mech-probe}, this pattern is consistent with diagnostic penalties targeting behaviors beyond invalid output.
The comparison is observational: panel~(a) measures only \method{}'s penalty allocation, and these single-seed trends do not isolate a causal effect of the penalties.

\subsection{Initial Prices and Online Price Dynamics}
\label{app:prices-move}
\label{app:beta-dynamics}
\paragraphhd{Cold-start prices.}
Table~\ref{tab:coldprices} lists the initial prices for ALFWorld with Qwen3-4B-Instruct-2507 (Appendix~\ref{app:inst}).
Eight of the $20$ error coefficients are negative and receive non-zero prices; the other twelve remain in the model with zero price.
Success is shorter in $1{,}085$ of $1{,}107$ within-task success--failure pairs and never longer.
The length coefficient reaches its lower bound, so initial error coefficients should be interpreted in light of this strong outcome--length association.

\begin{table}[!htbp]
\centering
\footnotesize
\caption{Cold-start prices on ALFWorld with Qwen3-4B-Instruct-2507. $\widetilde w_k=\max(-\beta_k,0)\,\sigma_k$ and $w_k=\widetilde w_k/\max_j\widetilde w_j$ (Eq.~\eqref{eq:prices}). The remaining 12 pricing classes have $\beta_k>0$ and $w_k=0$.}
\label{tab:coldprices}
\begin{tabular}{@{}lrrrr@{}}
\toprule
Pricing class & $\beta_k$ & $\sigma_k=\sd(x_{\cdot,k})$ & $\widetilde w_k$ & $w_k$ \\
\midrule
E7/\_other (non-executable output, other) & $-0.10791$ & $0.017713$ & $1.9115\times10^{-3}$ & $1.0000$ \\
E1.state\_requirement\_misread & $-0.07593$ & $0.015166$ & $1.1515\times10^{-3}$ & $0.6024$ \\
E2.failure\_cause\_misattribution & $-0.06719$ & $0.013516$ & $9.0816\times10^{-4}$ & $0.4751$ \\
E6/\_other (premature abandonment, other) & $-0.04978$ & $0.012275$ & $6.1109\times10^{-4}$ & $0.3197$ \\
E3.redundant\_state\_toggle\_repeat & $-0.02184$ & $0.010325$ & $2.2551\times10^{-4}$ & $0.1180$ \\
E1.state\_history\_loss & $-0.01125$ & $0.008723$ & $9.8129\times10^{-5}$ & $0.0513$ \\
E2.absence\_as\_evidence\_inference & $-0.00459$ & $0.013469$ & $6.1783\times10^{-5}$ & $0.0323$ \\
E1.command\_semantics\_misbelief & $-0.00045$ & $0.025040$ & $1.1284\times10^{-5}$ & $0.0059$ \\
\bottomrule
\end{tabular}
\end{table}

\paragraphhd{Tracking and deployment.}
We track coefficients in GRPO (diag-SFT), GiGPO (diag-SFT), the frozen-price variant and the full method, with GRPO + step penalty as a fifth reference in Figure~\ref{fig:beta_ema_grid}.
The tracker refits every ten updates from $30$ to $160$, yielding $14$ snapshots per run, and applies the sign-gated moving average in Eq.~\eqref{eq:signgate}.
For the two baselines, the tracker only monitors coefficients; the frozen-price variant records new fits but continues to charge the initial prices without the cap.
The full method deploys its updated prices, while the fifth reference adds explicit step penalties to GRPO (diag-SFT).

\paragraphhd{Class-specific changes.}
Figure~\ref{fig:beta_dynamics} shows three illustrative classes selected after inspecting the runs.
Search-state amnesia treats previously checked locations as unexplored; remote-interaction misbelief attempts to inspect or manipulate an object before reaching it; E7 residual groups unnamed errors within one broad family.
The full method's amnesia coefficient becomes more negative, from $-0.295$ at update $30$ to $-5.976$ at $160$; the other two end at $-1.034$ and $-1.571$.
Under GiGPO (diag-SFT), it declines until update $120$ and then partially recovers, showing that fitted associations also change without deploying prices.

Figure~\ref{fig:beta_ema_grid} reports all $20$ classes, including five that remain unchanged across runs.
Flat non-zero curves can reflect initial values retained by the sign gate, while large negative values can reflect fits approaching the coefficient bound; neither pattern establishes calibration quality.

\begin{figure}[!htbp]
\centering
\includegraphics[width=\textwidth]{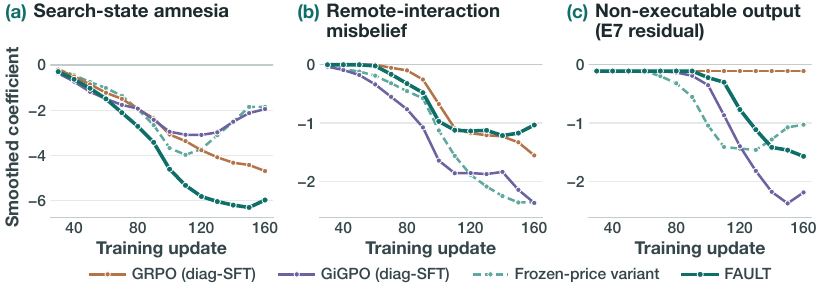}
\caption{\textbf{Coefficient dynamics for three illustrative error classes.} Sign-gated moving averages at $14$ refits per run; panels use different vertical scales. GRPO and GiGPO (diag-SFT) track coefficients without applying them, the frozen-price variant retains its initial prices, and the full method deploys updated prices. E7 residual collects unnamed errors within the non-executable-output family.}
\label{fig:beta_dynamics}
\end{figure}

\paragraphhd{Revising the initial price allocation.}
Search-state amnesia has a zero cold-start price because its initial coefficient is positive ($+0.18$), yet it is the most frequent pricing class in every tracker window of every run and receives the most negative coefficient at the first online refit.
At update $30$, it occurs in $70\%$ of the full method's tracker-window trajectories.
More generally, the rank agreement between initial prices and online weights declines across runs; for the frozen-price variant, it drops from $0.63$ at update $30$ to $-0.07$ at $160$ (Figure~\ref{fig:beta_price_dynamics}(a)).
The comparison also reflects different price mappings: cold-start prices rescale coefficients by feature standard deviations, whereas online prices use unit scaling (Appendix~\ref{app:pricing}).
The rankings thus describe changes in relative weights, not accuracy against known error costs.

\begin{figure}[H]
\centering
\includegraphics[width=\textwidth]{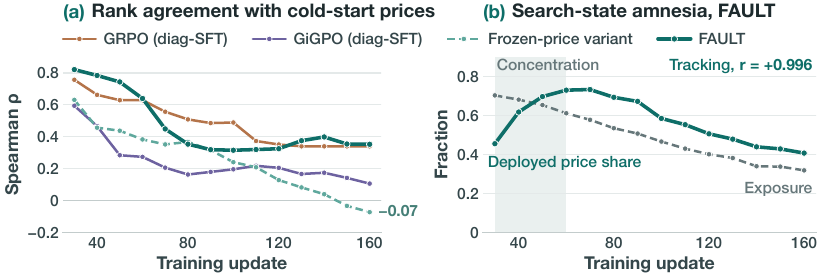}
\caption{\textbf{Changes in price rankings and relative price allocation.} (a) Spearman correlation between cold-start prices and each run's online weights. Only the full method deploys the updated weights; the other curves are monitoring estimates. (b) Search-state amnesia under the full method: its share of total deployed class weight (solid) and its occurrence rate among tracker-window trajectories (dashed). Shading marks updates $30$--$60$. The two series have Pearson correlation $+0.996$ over updates $70$--$160$; the class's own normalized weight remains $1$. Each curve represents $14$ refits from one run.}
\label{fig:beta_price_dynamics}
\end{figure}

\begin{figure}[p]
\centering
\includegraphics[width=\textwidth,height=0.86\textheight,keepaspectratio]{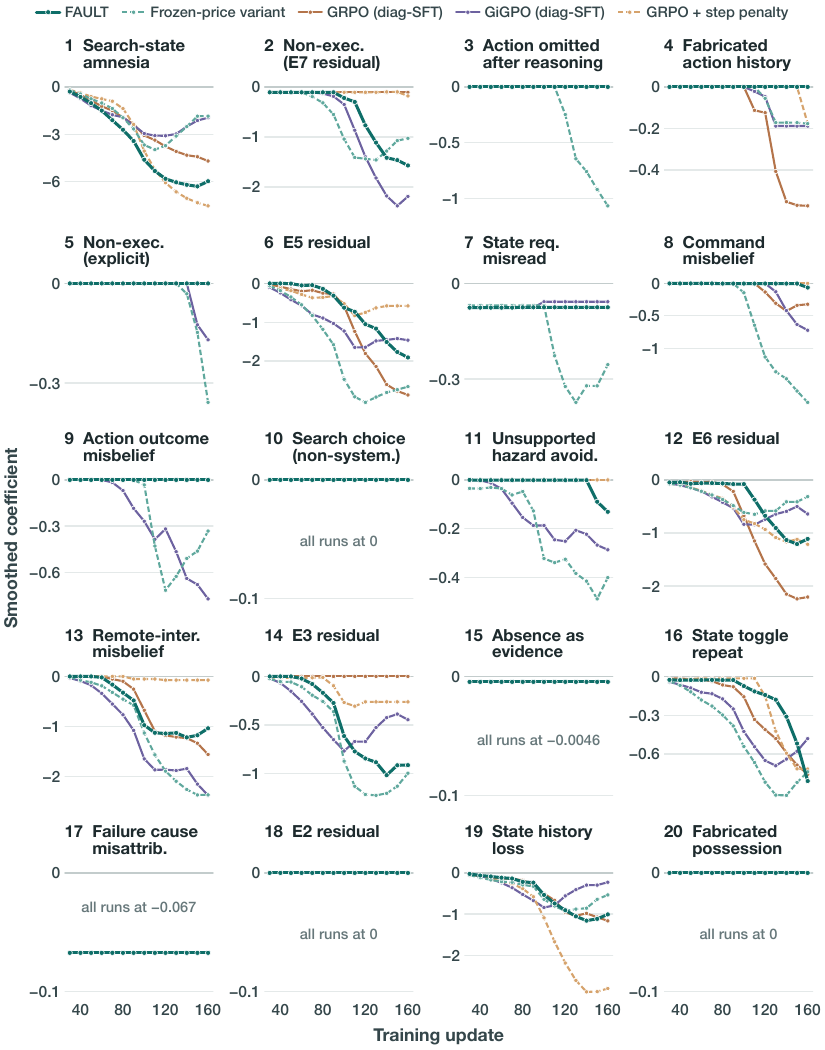}
\caption{\textbf{Coefficient dynamics for all 20 pricing classes.} Sign-gated moving averages at $14$ refits for five runs, with a separate vertical scale per panel. Coincident curves may be hidden by the full-method curve drawn last; several runs remain at zero in panels 3, 5 and 9. Residual classes collect errors not assigned to a named L2 class within an L1 family. Curves show estimator coefficients, not deployed prices.}
\label{fig:beta_ema_grid}
\end{figure}

Amnesia's share of total deployed class weight rises from $0.46$ at update $30$ to $0.73$ at $70$, then declines alongside its occurrence rate (Pearson $+0.996$ over updates $70$--$160$; Figure~\ref{fig:beta_price_dynamics}(b)).
Its own normalized weight remains $1$ throughout, so the falling share reflects increased weight on other classes, not a reduction in its own price.
From update $30$ to $160$, its occurrence rate falls by $55\%$ under the full method, compared with $34\%$ under GRPO (diag-SFT).
By update $160$, classes priced at zero initially account for $68\%$ of the full method's deployed class-weight mass.
These changes explain what online pricing adds to the initial allocation; their effect on policy performance is not isolated by these curves or the ablation configurations (Section~\ref{app:ablation-details}).

\paragraphhd{Interpretation.}
Each curve represents one run, and adjacent refits reuse data from overlapping $50$-update windows.
The coefficients describe associations that change with the policy, diagnoses and available samples, rather than causal error costs.
Raw fits are provided with the code.

\FloatBarrier
\section{Case Studies}
\label{app:cases}
We present two hand-selected ALFWorld training rollouts of \method{}, supplemented by baseline credit examples.
These single-seed cases illustrate the mechanisms; population statistics are reported in Appendices~\ref{app:exp-results} and~\ref{app:dynamics}.
Step numbers follow the zero-based training logs and judge labels: step~$8$ here corresponds to $t=9$ in \S\ref{sec:problem}.
As in Appendix~\ref{app:localization}, we compare \method{}'s diagnostic allocation $\delta_t$ with baseline negative optimizer advantages $\max(-A_t,0)$; $\delta_t$ is not \method{}'s full optimizer advantage.
The blind judge is an independent LLM that sees only actions and feedback; its labels are model judgments, not independently verified ground truth.

\subsection{Credit in an All-Fail Group}
\label{app:case}
\paragraphhd{Selection and group.}
Among the $300$ failed \method{} trajectories in Figure~\ref{fig:localization}(a), $82$ have their maximum diagnostic allocation at the judge's decisive step.
From these matches, we select the most concentrated all-fail case without invalid actions or no-ops.
Its task is \emph{clean some lettuce and put it in sidetable} (update $70$).
All eight rollouts fail at the $30$-step horizon, giving zero outcome-relative advantage, while their diagnostic budgets span $0.93$--$6.13$ reward units (Figure~\ref{fig:case_study}(a)).
The shared invalid-action reward can still provide a separate signal in all-fail groups.

\paragraphhd{Diagnostic allocation.}
In the selected trajectory ($P=3.18$), step~$8$ returns to the fridge after a five-location sweep, marking the first revisit of a searched location; the agent subsequently cycles through searched locations and never finds the lettuce (Table~\ref{tab:case_walkthrough}).
Approximately $82\%$ of the budget falls on step~$8$ and $18\%$ on the cabinet revisit at step~$14$ (Figure~\ref{fig:case_study}(b)).
The blind judge also selects step~$8$; its verdict is reproduced in the table.
Because this trajectory has no invalid actions or no-ops, its localization is not explained by mechanical flags.
The group remains at $0/8$ success: the case demonstrates budget contrast and agreement with the judge within an unsuccessful group.

\paragraphhd{Baseline comparison.}
Figure~\ref{fig:case_transcripts} includes independently sampled baseline rollouts from different task instances; no paired same-instance capture is available.
In the GRPO (diag-SFT) trajectory, step~$13$ revisits the fridge opened at step~$4$, and the blind judge identifies step~$13$ as decisive.
Within its all-fail group, outcome-relative advantage is zero, but format shaping yields two per-step advantage values, with invalid steps $0.1$ below valid steps.
Negative advantage consists of three equal $0.087$ spikes on unparseable steps and is zero at the judge's step.
Rescaling these values cannot distinguish that step from other format-valid steps.
The bottom strip shows an all-fail GiGPO (diag-SFT) trajectory whose negative advantage also falls only on format-flagged steps; no transcript is available because the original run did not store rollout text.
These examples illustrate differences in the recorded credit assignments, without establishing a paired performance comparison.

\begin{figure}[!htbp]
\centering
\includegraphics[width=\textwidth]{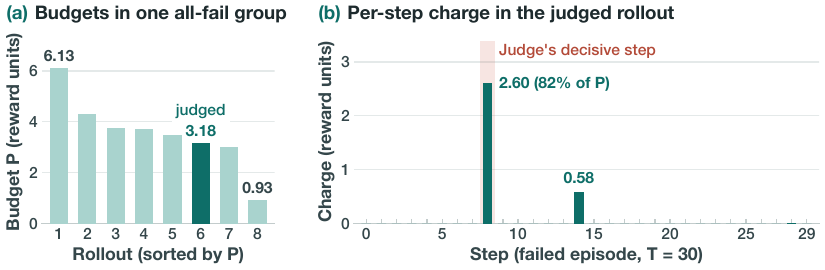}
\caption{\textbf{Diagnostic penalties within an all-fail group.} (a) Budgets for eight rollouts, all failing at $30$ steps. (b) Allocation for the selected trajectory: approximately $82\%$ at the judge's decisive step and $18\%$ at a later revisit (Table~\ref{tab:case_walkthrough}). The group contributes to Figure~\ref{fig:signal_over_training}; the case is selected for illustration.}
\label{fig:case_study}
\end{figure}

\begin{table}[!htbp]
\centering
\footnotesize
\caption{Walkthrough of the selected trajectory in Figure~\ref{fig:case_study}(b), abridged from the raw episode. Values of $\delta_t$ are rounded shares of the diagnostic budget $P{=}3.18$; the two largest allocations fall on revisits of searched locations.}
\label{tab:case_walkthrough}
\begin{tabular}{@{}c p{2.9cm} p{6.6cm} r@{}}
\toprule
Step & Action & Environment feedback (abridged) & $\delta_t$ \\
\midrule
0--1 & go to / open fridge 1 & fridge opened: a cup and a potato, no lettuce & $0$ \\
2--7 & first sweep & countertop 1; cabinets 2, 3 (opened, both empty); sinkbasin 1: four further locations, each visited for the first time & $0$ \\
\textbf{8} & \textbf{go to fridge 1} & \textbf{first revisit of a searched location; contents unchanged since step 1} & \textbf{2.60 (82\%)} \\
9--13 & search starts to cycle & countertop 1 and cabinet 2 revisited; cabinet 4 and diningtable 1 new & $0$ \\
\textbf{14} & \textbf{go to cabinet 3} & \textbf{revisit of a cabinet known to be empty since step 6} & \textbf{0.58 (18\%)} \\
15--28 & cycling continues & nine more steps on already-searched locations, four on new ones (diningtable 2, cabinet 1, sidetable 1), one \texttt{inventory} check (``not carrying anything'') & $0$ \\
29 & go to cabinet 1 & $30$-step limit reached; the episode fails; lettuce never found & $0$ \\
\midrule
\multicolumn{4}{@{}p{12.4cm}@{}}{\emph{Blind judge} (independent LLM; saw only this action--feedback text): decisive step $=8$; verdict: ``Never found lettuce; repeatedly revisited fridge, countertop and cabinets already seen instead of new spots.''} \\
\bottomrule
\end{tabular}
\end{table}

\begin{figure}[!htbp]
\centering
\includegraphics[width=\textwidth,height=0.86\textheight,keepaspectratio]{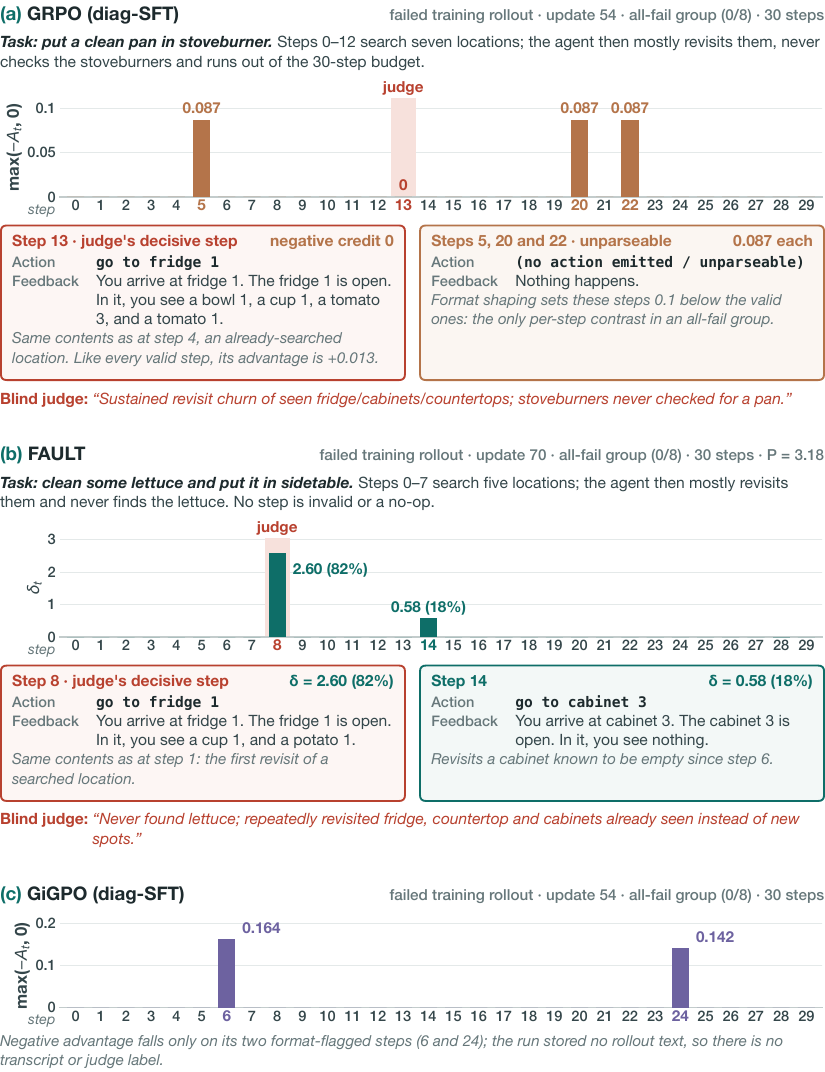}
\caption{\textbf{Transcripts and step penalties.} (a) GRPO (diag-SFT) and (b) \method{} failures from separate all-fail groups, each with its per-step credit: GRPO's negative optimizer advantage $\max(-A_t,0)$ and \method{}'s diagnostic allocation $\delta_t$. The blind judge selects steps $13$ and $8$, respectively (shaded); GRPO marks only three unparseable steps, while \method{} allocates approximately $82\%$ to the judged step. (c) Negative advantage for a GiGPO (diag-SFT) all-fail trajectory, marking only format violations; no transcript was stored. Action and feedback text is verbatim; italic notes are ours. Task instances differ.}
\label{fig:case_transcripts}
\end{figure}

\subsection{Recovery in a Successful Trajectory}
\label{app:case-transcripts}
\begin{figure}[!htbp]
\centering
\includegraphics[width=\textwidth,height=0.86\textheight,keepaspectratio]{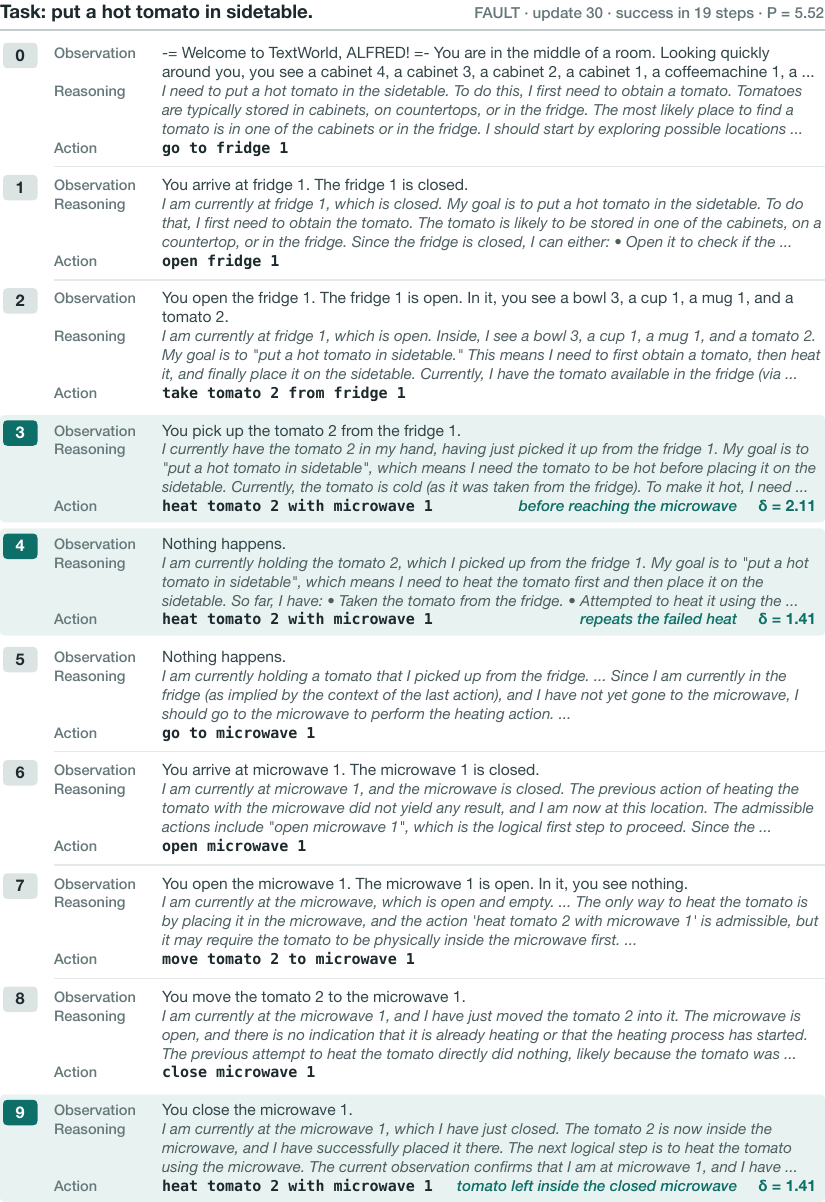}
\caption{\textbf{Recovery within a successful trajectory, steps 0--9.} The agent makes three ineffective heating attempts. Charge tags show the step allocation of budget $P=5.52$; the case continues in Figure~\ref{fig:case_alf_recovery_b}.}
\label{fig:case_alf_recovery}
\end{figure}
\begin{figure}[p]
\centering
\includegraphics[width=\textwidth,height=0.86\textheight,keepaspectratio]{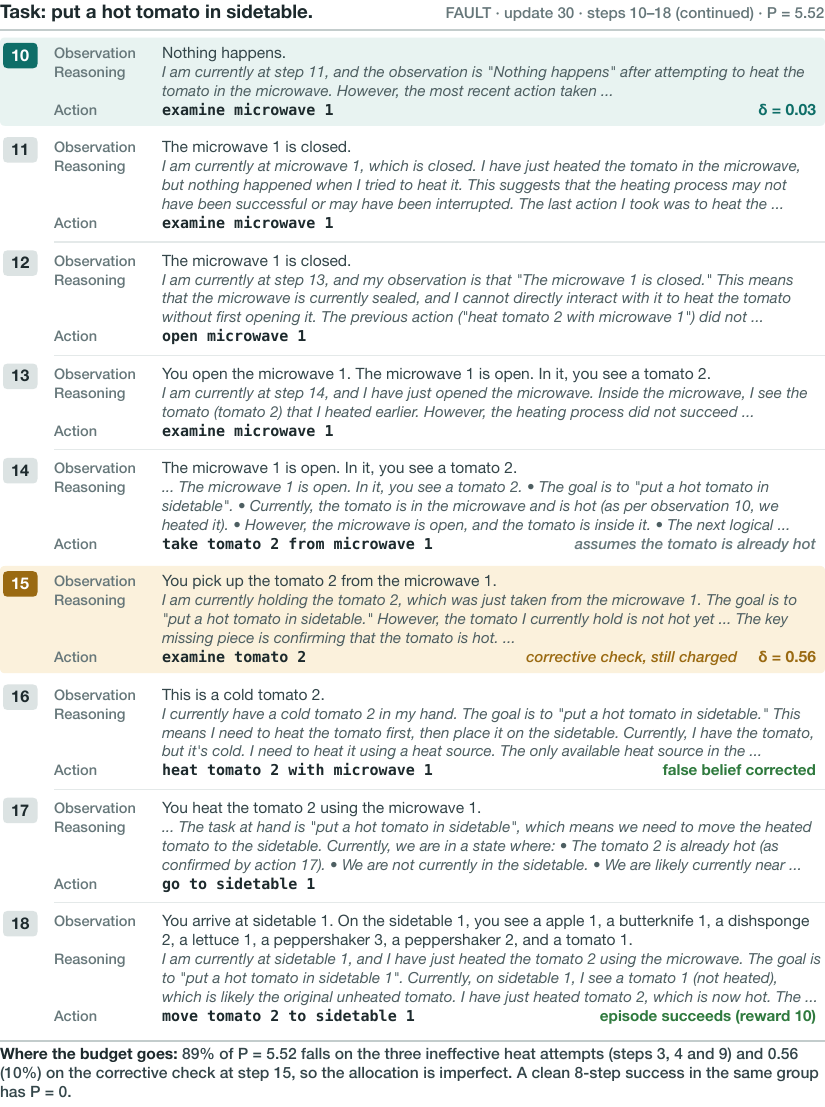}
\caption{\textbf{Recovery within a successful trajectory, steps 10--18.} Continued from Figure~\ref{fig:case_alf_recovery}. Examination reveals the tomato is still cold, after which the agent corrects its belief and completes the task. Most of the budget targets ineffective heating attempts, although the corrective examination also receives a penalty.}
\label{fig:case_alf_recovery_b}
\end{figure}
\paragraphhd{Errors and recovery.}
The second case is a $19$-step success at update $30$ on the task \emph{put a hot tomato in sidetable} (Figures~\ref{fig:case_alf_recovery} and~\ref{fig:case_alf_recovery_b}).
Three premature heating attempts return ``Nothing happens.'': two occur before reaching the microwave and one after closing its door.
The agent later assumes the tomato is hot, but the examination at step~$15$ returns ``This is a cold tomato'' in the next observation.
It then revises its belief at step~$16$, heats the tomato successfully and completes delivery.
Of the diagnostic budget $P=5.52$, approximately $89\%$ falls on the three ineffective attempts, but $0.56$ also falls on the corrective examination, showing that localization is imperfect.
A clean eight-step success in the same group has $P=0$.
The contrast illustrates how successful outcomes can coexist with different process errors and penalties.

\paragraphhd{Reading the transcript.}
Each step shows the input observation, generated reasoning and action.
Observations and actions follow the stored transcript, with long passages shortened by ``...''; reasoning is excerpted model-generated text with chat-template artifacts removed, not a statement of environment facts.
Charge tags show the diagnostic allocation $\delta_t$ of Eq.~\eqref{eq:redistribution-full}, and short notes beside the actions are ours.

\FloatBarrier
\section{Extended Related Work}
\label{app:related}
This section expands Section~\ref{sec:related} along the same three lines.

\paragraphhd{Outcome-based agentic RL.}
Agentic RL treats an LLM as a policy that acts over many turns and learns from the rewards it observes~\citep{zhang2025landscape}.
Following the success of GRPO for reasoning models~\citep{shao2024deepseekmath,guo2025deepseekr1}, it has been applied with GRPO, PPO or their variants to search~\citep{jin2025searchr1}, tool use~\citep{qian2025toolrl,feng2025retool}, web navigation~\citep{wei2025webagentr1}, and games and embodied text environments~\citep{wang2025ragen,feng2025gigpo}.
GRPO needs no critic: it normalizes terminal rewards across the rollouts of a task and assigns each rollout's advantage to every step.
This is simple and scales well, but the advantage of a step reflects the whole trajectory rather than that step.

Two consequences follow for long-horizon tasks.
With binary rewards, a group whose rollouts all succeed or all fail has zero advantage for every rollout; DAPO filters such groups out~\citep{yu2025dapo}, while \citet{yang2026progpo} give all-fail groups a fallback advantage from first-visit observation coverage.
In our ALFWorld runs, these same-outcome groups make up between $59\%$ and $64\%$ of training groups (Section~\ref{sec:analysis-signal}).
Within a failed rollout, every step receives the same credit, so the decisive mistake is not singled out.
Agent-specific methods such as ARPO~\citep{dong2025arpo} and GiGPO~\citep{feng2025gigpo} change how rollouts are sampled or compared, but their credit is still computed from outcomes.

\paragraphhd{Numeric process signals.}
Process supervision assigns rewards to intermediate steps.
In mathematical reasoning, step rewards from human labels~\citep{lightman2024verify}, automatic rollout labels~\citep{wang2024mathshepherd} or progress estimates~\citep{setlur2024rewarding} improve on outcome-only rewards.
Agentic RL has adopted the idea: process reward models for agents learn step values from Monte Carlo rollout returns~\citep{choudhury2025agentprm} or temporal-difference estimates~\citep{xi2025agentprm}, and recent methods score steps with prefix-aware reward models or the agent's own hindsight~\citep{kim2026pair,ma2026repro,yao2026mise}.
These signals are dense and can separate rollouts that share an outcome.

Their difficulty is trust.
A learned step reward is itself a model that the policy can exploit; DeepSeek-R1 avoids neural process rewards partly for this reason~\citep{guo2025deepseekr1}, and further optimizing an agent PRM can lower the true success rate~\citep{choudhury2025agentprm}.
GiGPO avoids a learned model by comparing actions taken from the same environment state~\citep{feng2025gigpo}.
Its outcome-based credit, however, comes from discounted returns, so it vanishes in all-fail groups and, when present, spreads over the steps after a wrong turn (Section~\ref{sec:analysis-local}).
\method{} keeps the terminal reward as the only trusted signal: step penalties come from verified diagnoses, and their prices are a few coefficients fitted to outcomes.

\paragraphhd{Language analysis of trajectories.}
Language methods analyze completed trajectories to state what went wrong and where.
The analysis can guide another attempt through context or memory~\citep{shinn2023reflexion,zhao2024expel,hu2025echo,yao2024retroformer,ge2025samule}, earn a reward when a retry succeeds~\citep{bensal2025reflect}, produce corrected training trajectories~\citep{yuan2025agentr,zhang2025critiquegrpo,li2026critic}, or be distilled into the policy~\citep{yang2026opid,zheng2026grsd,bi2026reflectrl,wu2026seed}.
SEED is closest to our setting: one checkpoint acts, analyzes its rollouts into hindsight skills and learns from them by on-policy distillation, with strong gains on ALFWorld~\citep{shridhar2021alfworld} and WebShop~\citep{yao2022webshop}.
Such analysis carries the information that outcome credit lacks.

Three gaps keep it from solving the two credit-assignment problems above.
First, the analysis rarely sets RL credit: it serves as context, a training target or an object of reward, even when pooled across trajectories as in SAMULE~\citep{ge2025samule}.
GRSD~\citep{zheng2026grsd} does use pooled reflections to rescale each step's outcome advantage, but this advantage is still zero in same-outcome groups, so equal outcomes still receive equal outcome credit.
Second, error claims are not checked at the stated step; on TRAIL, the best evaluated model reaches only about $11\%$ joint accuracy in identifying and locating errors in agent traces~\citep{deshpande2025trail}, even though error taxonomies are available~\citep{cemri2025multiagent,zhu2025agentfail}.
Third, the weight of the analysis is set by fixed loss coefficients, mixing schedules or a single retry, not learned from outcomes.
\method{} closes these gaps with diagnoses restricted to a fixed taxonomy and checked by quotes (Appendix~\ref{app:diag}), prices fitted to outcomes within each task (Appendix~\ref{app:pricing}), and a bounded penalty budget redistributed over verified error steps.

\end{document}